\documentclass[acmsmall]{acmart}

\AtBeginDocument{%
  \providecommand\BibTeX{{%
    \normalfont B\kern-0.5em{\scshape i\kern-0.25em b}\kern-0.8em\TeX}}}

\setcopyright{acmcopyright}
\copyrightyear{2022}
\acmYear{2022}

\usepackage{microtype}
\usepackage[T1]{fontenc}    
\usepackage{url}            
\usepackage{booktabs}       
\usepackage{amsfonts}       
\usepackage{nicefrac}       
\usepackage{amsthm}
\usepackage{graphicx}
\usepackage{subfigure} 
\usepackage{bm}
\usepackage{hyperref}
\usepackage{amsmath}

\usepackage{amssymb}
\usepackage{cleveref}
\usepackage{enumitem}
\usepackage{mathtools}
\usepackage{array}
\usepackage{lipsum}
\usepackage{hyperref}
\usepackage{multirow}
\usepackage{makecell}
\usepackage{xcolor}
\usepackage{natbib}
\usepackage{tikz}
\usepackage{etoolbox}
\newtheorem{example}{Example}

\usepackage{algorithm}
\usepackage[noend]{algpseudocode}

\usepackage{braket}

\usepackage{afterpage}
\usepackage{makecell}
\usepackage{tabulary}
\usepackage{xcolor}
\usepackage{colortbl}
\usepackage{prettyref}
\usepackage{framed}
\usepackage{booktabs}       
\newcommand{\pref}[1]{\prettyref{#1}}

\newcommand{\savehyperref}[2]{\texorpdfstring{\hyperref[#1]{#2}}{#2}}
\newrefformat{eq}{\savehyperref{#1}{Eq. \textup{(\ref*{#1})}}}
\newrefformat{eqn}{\savehyperref{#1}{Eq.~\textup{(\ref*{#1})}}}
\newrefformat{lem}{\savehyperref{#1}{\textbf{Lemma~\ref*{#1}}}}
\newrefformat{assump}{\savehyperref{#1}{\textbf{Assumption~\ref*{#1}}}}

\newrefformat{defi}{\savehyperref{#1}{Definition~\ref*{#1}}}\newrefformat{tab}{\savehyperref{#1}{Table~\ref*{#1}}}
\newrefformat{lemma}{\savehyperref{#1}{Lemma~\ref*{#1}}}
\newrefformat{line}{\savehyperref{#1}{Line~\ref*{#1}}}
\newrefformat{thm}{\savehyperref{#1}{\textbf{Theorem~\ref*{#1}}}}
\newrefformat{corr}{\savehyperref{#1}{Corollary~\ref*{#1}}}
\newrefformat{cor}{\savehyperref{#1}{Corollary~\ref*{#1}}}
\newrefformat{sec}{\savehyperref{#1}{Section~\ref*{#1}}}
\newrefformat{app}{\savehyperref{#1}{Appendix~\ref*{#1}}}
\newrefformat{ex}{\savehyperref{#1}{Example~\ref*{#1}}}
\newrefformat{fig}{\savehyperref{#1}{Figure~\ref*{#1}}}
\newrefformat{alg}{\savehyperref{#1}{\textbf{Algorithm~\ref*{#1}}}}
\newrefformat{rem}{\savehyperref{#1}{Remark~\ref*{#1}}}
\newrefformat{conj}{\savehyperref{#1}{Conjecture~\ref*{#1}}}
\newrefformat{prop}{\savehyperref{#1}{Proposition~\ref*{#1}}}
\newrefformat{proto}{\savehyperref{#1}{Protocol~\ref*{#1}}}
\newrefformat{prob}{\savehyperref{#1}{Problem~\ref*{#1}}}
\newrefformat{claim}{\savehyperref{#1}{Claim~\ref*{#1}}}
\newrefformat{que}{\savehyperref{#1}{Question~\ref*{#1}}}
\newrefformat{op}{\savehyperref{#1}{Open Problem~\ref*{#1}}}
\newrefformat{fn}{\savehyperref{#1}{Footnote~\ref*{#1}}}
\newrefformat{property}{\savehyperref{#1}{Property~\ref*{#1}}}

\usepackage{xspace}
\usepackage{amsmath}
\usepackage{amsthm}
\usepackage{bbm}
\usepackage{mathrsfs}
\usepackage{bm}
\usepackage{makecell}
\usepackage{tabulary}

\DeclareMathOperator{\TV}{TV}

\DeclareMathOperator*{\argmax}{argmax} 

\newcommand{\T}{\mathrm{T}}

\newcommand{\calD}{\mathcal{D}}
\newcommand{\calL}{\mathcal{L}}
\newcommand{\UU}{{\mathcal{U}_k}}
\newcommand{\calN}{\mathcal{N}}

\newcommand{\calS}{\mathcal{S}}
\newcommand{\calM}{\mathcal{M}}

\newcommand{\calW}{\mathcal{W}}

\newcommand{\calRO}{(k)}
\newcommand{\calRRO}{}
\newcommand{\R}{\mathbb{R}}

\newcommand{\E}{\mathbb{E}}

\newcommand{\para}{W}
\newcommand{\what}{\widehat}

\newcommand{\one}{\mathbbm{1}}

\newcommand{\tildewt}{\widetilde{W}_t^{(k)}}

\newcommand{\distpara}{\widetilde{W}^{(k)}_t}

\newcommand{\wtilde}{\widetilde}

\newcommand{\blue}[1]{{\color{black}#1}}

\newcommand{\mathsc}[1]{{\normalfont\textsc{#1}}}
\newtheorem*{theorem*}{Theorem}

\newtheorem{assumption}{Assumption}[section]

\newtheorem*{lemma*}{Lemma}
\newtheorem{definition}{Definition}[section]

\theoremstyle{definition}

\newtheorem{thm}{Theorem}[section]

\newtheorem{lem}[thm]{Lemma}

\begin{document}

\title{Towards Achieving Near-optimal Utility for Privacy-Preserving Federated Learning via Data Generation and Parameter Distortion}

\author{Xiaojin Zhang}
\email{xiaojinzhang@ust.hk}
\affiliation{%
  \institution{Hong Kong University of Science and Technology}
  \streetaddress{Clear Water Bay}
  \country{China}
}

\author{Kai Chen}
\email{kaichen@cse.ust.hk}
\affiliation{%
  \institution{Hong Kong University of Science and Technology}
  \streetaddress{Clear Water Bay}
  \country{China}
}


\author{Qiang Yang}
\email{qyang@cse.ust.hk}
\authornote{Corresponding author}
\affiliation{%
  \institution{WeBank and Hong Kong University of Science and Technology}
  \country{China}
}

\renewcommand{\shortauthors}{Trovato and Tobin, et al.}

\begin{abstract}
     Federated learning (FL) enables participating parties to collaboratively build a global model with boosted utility without disclosing private data information. Appropriate protection mechanisms have to be adopted to fulfill the requirements in preserving \textit{privacy} and maintaining high model \textit{utility}. The nature of the widely-adopted protection mechanisms including \textit{Randomization Mechanism} and \textit{Compression Mechanism} is to protect privacy via distorting model parameter. We measure the utility via the gap between the original model parameter and the distorted model parameter. We want to identify under what general conditions privacy-preserving federated learning can achieve near-optimal utility via data generation and parameter distortion. To provide an avenue for achieving near-optimal utility, we present an upper bound for utility loss, which is measured using two main terms called variance-reduction and model parameter discrepancy separately. Our analysis inspires the design of appropriate protection parameters for the protection mechanisms to achieve near-optimal utility and meet the privacy requirements simultaneously. The main techniques for the protection mechanism include parameter distortion and data generation, which are generic and can be applied extensively. Furthermore, we provide an upper bound for the trade-off between privacy and utility, \blue{which together with the lower bound provided by no free lunch theorem in federated learning (\cite{zhang2022no}) form the conditions for achieving optimal trade-off.}

\end{abstract}

\ccsdesc[500]{Security and privacy}
\ccsdesc[500]{Computing methodologies~\text{Artificial Intelligence}}
\ccsdesc[100]{Machine Learning}
\ccsdesc[100]{Distributed methodologies}

\keywords{federated learning, privacy, utility, trade-off, optimization}

\maketitle
\section{Introduction} 
The popularity of distributed learning has grown as a result of the expansion of massive data sets. Data possessed by one company is not permitted to be shared to others due to the enforcement of data privacy laws like the General Data Protection Regulation (GDPR). Federated learning (FL)~\cite{mcmahan2016federated,mcmahan2017communication,konevcny2016federated,konevcny2016federated_new} meets this requirement by allowing multiple parties to train a machine learning model collaboratively without sharing private data. In recent years, FL has achieved significant progress in developing privacy-preserving machine learning systems.

We consider a \textit{horizontal federated learning} (HFL) setting. A total of $K$ clients upload respective local models to the server, who is responsible for aggregating multiple local models into a global model. There are a variety of application scenarios that use this scheme for federated learning (\cite{mcmahan2016federated, mcmahan2017communication, yang2019federated, yang2019federated_new, jebreel2022enhanced}). Although the private data of each client is not shared with other collaborators, it was revealed that exposed gradients of learnt models could be used by semi-honest adversaries to recover private training images with pixel-level accuracy (e.g., DLG \cite{zhu2020deep}, Inverting Gradients \cite{geiping2020inverting}, Improved DLG \cite{zhao2020idlg}, GradInversion \cite{yin2021see}), referred to as the gradient leakage attacking. Lots of attacking mechanisms reconstructed the private data via minimizing a candidate image in relation to a loss function that gauges the separation between the shared and candidate gradients. From an information theory point of view, the amount of information about private data that a semi-honest party can infer from exchanged information is inherently determined by the \textit{statistical dependency} between private data and publicly exchanged information. The semi-honest adversaries (\cite{zhu2019dlg,zhu2020deep,he2019model}) can exploit this dependency to recover the private training images with pixel-level accuracy from exchanged gradients of learned models. Preserving privacy is of immense practical importance when federating across different parties. Keeping potential privacy leakage at a manageable level is a crucial necessity for sustaining privacy.

The fundamental requirement on privacy-preserving federated learning (PPFL) is to maintain potential \textit{privacy leakage} below an acceptable level. To protect private data of the participants, many protection mechanisms have been proposed, such as \textit{Randomization Mechanism} \cite{geyer2017differentially,truex2020ldp,abadi2016deep}, \textit{Secret Sharing} \cite{SecShare-Adi79,SecShare-Blakley79,bonawitz2017practical}, \textit{Homomorphic Encryption (HE)} \cite{gentry2009fully,batchCryp}, and \textit{Compression Mechanism} \cite{nori2021fast}. The essence of these protection mechanisms is to distort the exchanged model parameter. For example, \textit{Randomization Mechanism} adds noise that follows some predefined distributions on the model parameter, and \textit{Compression Mechanism} distorts the original model parameter to the extent that some dimensions are eliminated. 

The distorted model parameter might make the aggregated model less accurate and result in a positive amount of \textit{utility loss}, as compared to model training without distorting model parameter. \citet{zhang2022no} proposed the No Free Lunch theorem (NFL) that builds a unified framework to depict the relationship between privacy and utility in federated learning. The privacy and utility are measured via distortion extent, a metric that quantifies the difference of data distributions before and after privacy protection. NFL provides a lower bound for the weighted summation of privacy leakage and utility loss. A natural question comes out: is it possible to achieve near-optimal utility subject to the requirement on privacy leakage? In this work, we provide an affirmative answer for a special form of measurement for utility. We further derive an upper bound for the trade-off between privacy and utility (see \pref{thm: upper_bound_trade_off_mt}). These two theoretical bounds together lead to the optimal trade-off between privacy and utility. 


\subsection{Our Contribution}
We are interested in analyzing the consistency between generalization and privacy-preserving. The utility loss of client $k$ (denoted as $\epsilon_{u,k}$) measures the variation in utility of client $k$ with the federated model drawn from unprotected distribution and the utility of the federated model drawn from protected distribution. To provide an avenue for achieving near-optimal utility, we first provide an upper bound for utility loss, which is measured using two main terms called variance-reduction and model parameter discrepancy separately. With the constraint on privacy leakage, the model parameter discrepancy is then determined. The upper bound on utility loss can be set to zero via adjusting the sampling probability appropriately, resulting in near-optimal utility.

\begin{itemize}
    \item  To derive the bound for utility loss, we use bias-variance decomposition, which could be interpreted as generalization-risk decomposition. The upper bound for utility loss (\pref{thm: upper_bound_of_epsilon_u}) measures the trade-off between variance-reduction and model parameter discrepancy.
    \item With the requirement on privacy leakage, we can determine the least amount of distortion extent (\pref{lem: from_privacy_to_distortion}). Given the total variation distance, we can derive the variance of the added noise according to \pref{lem: total_variation_and_variance}. The utility is further influenced by the sampling probability $p$ that is used for constructing the mini-batch (\pref{thm: from _utility_to_probability}). 
    \item Inspired by the theoretical analyses, we design an algorithm that achieves near-optimal utility and simultaneously satisfies the requirements on privacy leakage. The whole algorithm is illustrated in \pref{alg: adaptive learning algorithm}.
    \item We provide an upper bound for the weighted summation of privacy leakage and utility loss (see \pref{thm: upper_bound_trade_off_mt}). This bound together with the bound shown in NFL (\cite{zhang2022no}) form the optimal trade-off between privacy and utility. This theorem informs us how to achieve optimal privacy-utility trade-offs in \pref{thm: optimal_trade_off}, and implies the conditions when the proposed mechanisms achieve the optimal utility loss given the privacy leakage, and also provides an avenue for achieving the optimal privacy leakage given the utility loss.
\end{itemize}
\section{Related Work}
\paragraph{Attacking Mechanisms in Federated Learning}\label{sec:related:attack}
We focus on \textit{semi-honest} adversaries who faithfully follow the federated learning protocol but may infer private information of other participants based on exposed model information. In HFL, \citet{zhu2019dlg,zhu2020deep,geiping2020inverting,zhao2020idlg,yin2021see} demonstrate that adversaries could exploit gradient information to restore the private image data to pixel-level accuracy,with distinct settings of prior distributions and conditional distributions.

\paragraph{Protection Mechanisms in Federated Learning}
A variety of protection mechanisms have been proposed in HFL to prevent private data from being deduced by adversarial participants,and the most popular ones are \textit{Homomorphic Encryption (HE)}~\cite{gentry2009fully,batchCryp}, \textit{Randomization Mechanism}~\cite{geyer2017differentially,truex2020ldp,abadi2016deep}, \textit{Secret Sharing}~\cite{SecShare-Adi79,SecShare-Blakley79,bonawitz2017practical} and \textit{Compression Mechanism} \cite{nori2021fast}. Another school of FL~\cite{gupta2018distributed,gu2021federated} tries to protect privacy by splitting a neural network into private and public models,and sharing only the public one~\cite{kang2021privacy,gu2021federated}.

\paragraph{Model Accuracy}
\citet{sajadmanesh2021locally} introduce how to find a suitable parameter to minimize the variance, the relationship between variance reduction and utility loss is not considered. \citet{kaya2021does} show that label smoothing can increase accuracy and protection at the same time. \citet{de2022mitigating} introduce the use of data augmentation, which includes higher accuracy on unseen clients, mitigate data heterogeneity, and much sparser communication. 

\paragraph{Privacy-Utility Trade-off}
In the past decade,there has been wide interest in understanding utility-privacy trade-off. \citet{sankar2013utility} quantified utility via accuracy,and privacy via entropy. They provided a utility-privacy tradeoff region for i.i.d. data sources with known distribution based on rate-distortion theory. They left the problem of quantifying utility-privacy tradeoffs for more general sources as a challenging open problem. \citet{makhdoumi2013privacy} modeled the utility-privacy tradeoff according to the framework proposed by \citet{du2012privacy}. They regard the tradeoff as a convex optimization problem. This problem aims at minimizing the log-loss by the mutual information between the private data and released data,under the constraint that the average distortion between the original and the distorted data is bounded. \citet{reed1973information,yamamoto1983source,sankar2013utility} provided asymptotic results on the rate-distortion-equivocation region with an increasing number of sampled data. \citet{du2012privacy} modeled non-asymptotic privacy guarantees in terms of the inference cost gain achieved by an adversary through the released output. The theoretical analysis of the privacy-utility trade-off within the privacy-defense framework is presented in works by Zhang et al. (2022) \cite{zhang2022no}, Zhang et al. (2023) \cite{zhang2023trading}, and Zhang et al. (2023) \cite{zhang2023probably}. Furthermore, optimized algorithms for balancing privacy and utility under the same framework have been designed as detailed in the studies by Zhang et al. (2023) \cite{zhang2023theoretically}. Additionally, from a game-theoretic perspective, Zhang et al. (2023) \cite{zhang2023game} provide strategic insights for the privacy-utility trade-off, offering a nuanced understanding of the adversarial dynamics involved.

\section{Preliminaries}
We focus on the HFL setting, consisting of a total of $K$ clients and a server. We denote $\calD^{(k)}$ as the dataset owned by client $k$, and $|\calD^{(k)}|$ as the size of the dataset of client $k$. Let $\calL^{(k)}(W) = \frac{1}{|\calD^{(k)}|}\sum_{i = 1}^{|\calD^{(k)}|} \calL(W, d_i^{(k)})$ be the loss of predictions made by the model parameter $W$ on dataset $\calD^{(k)}$, where $d_i^{(k)}$ represents the $i$-th data-label pair of client $k$. Let $W^*$ denotes the optimal model parameter that minimizes the federated loss. The objective of the clients is to collaboratively train a global model:
    \begin{align*}
        W^* &= \arg\min_{W}\sum_{k = 1}^K \frac{|\calD^{(k)}|}{\sum_{k=1}^K |\calD^{(k)}|}\ \calL^{(k)}(W).
    \end{align*}



\begin{definition}[The form of sum-of-squres]\label{defi: sum_of_squres}
     Assume the upper bound of the loss function $\calL$ has the form of sum-of-squares. More specifically, we assume there exists a constant $C > 0$, satisfying that
   \begin{align}
      \calL(W_{t}^{\calRO})\le\text{GAP}(W_{t}^{\calRO}) = C\cdot\|W_{t}^{\calRO} - W^{*}\|^2.  
   \end{align}
\end{definition}

\begin{example}[Loss Function with the Form of Sum-of-Squares]\label{ex: loss_func_sum_of_squares}
   Let $X = (X_1, \cdots, X_d)$, and $W = (W_1, \cdots, W_d)$. Let $W_{t}^{\calRO} = W_{t-1}^{\calRO} - \frac{1}{|\calD^{(k)}|} \sum_{i\in\calD^{(k)}} \nabla\calL(W_{t-1}^{\calRO}, d_i^{(k)})$ represent the model parameter at round $t$, which is updated using the mini-batch from client $k$. 
Let $W^*$ denote the optimal model parameter, i.e., the parameter satisfying that $W^* = \arg\min_{W}\frac{1}{N} \sum_{i = 1}^N \calL(W, d_i^{(k)})$. Let $M$ represent the data size. Then we have
\begin{align*}
    \calL(W_{t}^{\calRO})  = (\sum_{j = 1}^d X_j W_j^* - \sum_{j = 1}^d X_j W_{t,j}^{\calRO})^2 & \le \sum_{j = 1}^d X_j^2 \cdot\sum_{j = 1}^d \left(W_j^* - W_{t,j}^{\calRO}\right)^2\\
    & = \|W_{t}^{\calRO} - W^{*}\|^2,
\end{align*}
where the inequality is due to Cauchy-Schwarz inequality.
\end{example}
The above example motivates us to define the utility loss as follows.

\begin{definition}[Utility Loss]\label{defi: utility_loss}
Let $\epsilon_{u,t}^{(k)}$ represent the utility loss of client $k$ at round $t$, which is defined as
\begin{align}
   \epsilon_{u,t}^{(k)} & = \text{GAP}(W_{t}^{\calRO}) - \text{GAP}(\tildewt)\\
   & = \|W_{t}^{\calRO} - W^{*}\|^2 - \|\tildewt - W^{*}\|^2.
\end{align}
The utility loss of the federated system is the average utility loss over rounds and clients, 
\begin{align}
    \epsilon_u = \frac{1}{K}\frac{1}{T}\sum_{k = 1}^K \sum_{t = 1}^T \epsilon_{u,t}^{(k)}.
\end{align}
\end{definition}
\textbf{Remark:} The utility loss measures the gap between the utility of the true model parameter and that of the distorted model parameter. We consider a special class of loss function, which could be approximated using $\|W_{t}^{\calRO} - W^{*}\|^2$.

The privacy leakage measures the discrepancy between the adversaries' belief with and without leaked information. Moreover, the privacy leakage is averaged with respect to the protected model information variable which is exposed to adversaries.
\begin{definition}[Privacy Leakage]\label{defi: average_privacy_JSD}
Let $\wtilde F_t^{(k)}$ represent the belief of client $k$ about the private information after observing the protected parameter. Let $\epsilon_{p,t}^{(k)}$ represent the privacy leakage of client $k$ at round $t$, which is defined as
\begin{align}\label{eq: def_of_pl}
\epsilon_{p,t}^{(k)} = \sqrt{{\text{JS}}(\wtilde F_t^{(k)} || F_t^{(k)})}.
\end{align}
Furthermore, the privacy leakage in FL resulted from releasing the protected model information is defined as
\begin{align}
\epsilon_{p,t} = \frac{1}{K}\sum_{k=1}^K \epsilon_{p,t}^{(k)}.
\end{align} 
\end{definition}


\section{Privacy-Preserving FL Framework}
In this section, we introduce the framework for the protection and the attacking mechanisms.

\subsection{Threat Model}
We consider the scenario where the server is a semi-honest attacker. The attacker is honest-but-curious. He/she adheres to the algorithm, and may infer the private information of the clients upon observing the uploaded information. We essentially follow the commonly used data reconstruction attacking model (\cite{zhu2019dlg}). The attacker is aware of the following information:
\begin{itemize}
    \item Machine learning model $F$;
    \item Model parameter uploaded to the server;
    \item The average gradient calculated using a collection of M training samples;
    \item The size of the mini-batch;
    \item Label information (optional).
\end{itemize}

The semi-honest attacker is aware of the label information $\{Y_1, \cdots, Y_m\}$, upon observing the distorted model parameter $\wtilde W$, he infers the feature information $\{\wtilde X_1, \cdots, \wtilde X_m\}$. Notice that the machine learning model $F$ and model parameter $W$ with respect to which the gradient is calculated are known to the adversary.

\subsection{Protection Mechanism}
FedAvg and FedSGD, two theoretically comparable representative aggregation implementations from HFL, are covered by our framework.

\blue{The protector obtains a mini-batch from his dataset. The mini-batch is denoted as $\calD = \{\textbf{X}, \textbf{Y}\} = \{(X_1,Y_1),\cdots, (X_m,Y_m)\}$. The protector generates the true gradient $\nabla W$ using $\calD$: $\frac{\partial \calL(F(\textbf{X}, W), \textbf{Y})}{\partial W}$ and uploads the distorted gradient $\wtilde{\nabla W}$.} Now we elaborate on three main procedures in detail. 

\paragraph{Step 1: Mini-Batch Generation}
Let $M$ represent the total size of the dataset, and $N$ represent the total number of rounds for sampling. Each data $d\in\calD^{(k)}$ is sampled with probability $p$, and thereby obtaining the \textit{mini-batch}, denoted as $\calS^{(k)}$. This is also regarded as the private data the defender aims to protect.

\paragraph{Step 2: Parameter Optimization}

\blue{

With the global model sent from the server, each client $k$ updates local model parameter $\para_t^{\calRO}$ using stochastic gradient descent with its own data set $\calD^{(k)}$, and the updated model parameter of client $k$ is denoted as $W_t^{(k)}$.

We follow the update rule of stochastic gradient descent used by \textit{federated SGD} (FedSGD) (\cite{mcmahan2017communication}). The model parameters at round $t+1$ are updated as:
}
    
    \begin{align}
    W_{t+1}^{(k)}\leftarrow W_{t} - \frac{\eta_t}{|\calS_t^{(k)}|}\sum_{i\in \calS_t^{(k)}} \nabla \calL^{(k)}_t(W_{t}, d_i^{(k)}),
\end{align}
where $W_t$ represents the federated model parameter at round $t$, and $\eta_t$ represents the learning rate.

Viewing the above process, we know that the dataset $\calS_t^{(k)}$ is mapped to a model parameter via stochastic gradient descent (SGD). As a result, the model parameter is related with the gradient of the loss on the dataset $\calS_t^{(k)}$. This mapping is deterministic once $\calS_t^{(k)}$ and the initial model parameter $W_t$ are fixed.

\paragraph{Step 3: Parameter Distortion}
We now introduce federated learning procedures that preserve privacy via distorting the model parameter. The protection mechanism $\calM$ is defined as $\calM: \calW\rightarrow \calW$, where $\calW$ represents the domain of the model parameter. The updated model parameter $W_{t+1}^{(k)}$ is distorted as $\widetilde W_{t+1}^{(k)}$, and is then uploaded to the server,
    \begin{align}
        \widetilde W_{t+1}^{(k)} = W_{t+1}^{(k)} + \delta_{t+1}^{(k)}.
    \end{align}


The server aggregates the received model parameters from the clients as an aggregated model parameter,
        \begin{align}
            \wtilde W_{t+1} = \frac{1}{K}\sum_{k = 1}^K \widetilde W_{t+1}^{(k)}.
        \end{align}








\section{Theoretical Analysis}
In this section, we introduce our main theorem (\pref{thm: upper_bound_of_epsilon_u}), which provides an avenue for achieving near-optimal utility. We provide upper bounds for utility loss and privacy leakage using sum of squares and bias-variance decomposition.

To derive the bounds for utility loss, we need the following assumptions.

\blue{
The following assumption means that the norm of every model parameter in our considered set is limited to a maximum value.
\begin{assumption}\label{assump: upper_bounded_by_c_3}
    Assume that $\|W\|\in [0, C_3]$ for any $W \in \mathcal{W}^{(k)}$.
\end{assumption}

This assumption deals with the average distance between the model parameters and the best possible parameter. It tells us that, on average, these parameters are not too far from the ideal one, and there's a maximum distance limit.
\begin{assumption}\label{assump: upper_bounded_by_c_4}
    Assume that $\|\E[W] - W^*\|\in [0, C_4]$ for any $W \in \mathcal{W}^{(k)}$. 
\end{assumption}

In short, both assumptions set boundaries for the model parameters: the first limits how large they can be, and the second limits how far they might stray from the best case. These boundaries help ensure our model performs well, avoiding situations where parameters are overly large or deviate too much from the desired outcome.
}

\subsection{Upper Bound for Utility Loss}
Let $P_t^{(k)}$ represent the distribution of $W_{t}^{\calRO}$, and $\wtilde P_t^{(k)}$ represent the distribution of $\distpara$, then $\text{TV}(P_t^{(k)}, \wtilde P_t^{(k)})$ represents the distance between the distributions of $W_{t}^{\calRO}$ and $\distpara$. 

The following theorem shows that the utility loss is bounded by the distance between the protected and unprotected distributions.
The distribution of the distorted model parameter $\wtilde P_t^{(k)}$ and that of the original model parameter $P_t^{(k)}$ are distinct, and lead to a certain level of bias. Please refer to \pref{sec: upper_bound_of_epsilon_u} for the full proof.


\begin{thm}\label{thm: upper_bound_of_epsilon_u}
Let $\epsilon_{u,t}^{(k)}$ be defined in \pref{defi: utility_loss}, then we have that
\begin{align}\label{eq: bound_for_utility_loss}
    \epsilon_{u,t}^{(k)} \le -\E(\text{Var}[\tildewt| W_{t - 1}^{(k)}]) + C_6\cdot {\text{TV}}(P^{(k)}_{t} || \wtilde P^{(k)}_{t} ),
\end{align}
where the first term is related to generalization in the stochastic gradient descent procedure, and the second term is related to the protection mechanism.
\end{thm}
\textbf{Remark:}
The upper bound for utility loss informs us that under some circumstances, the utility will not decrease but instead will increase. The performance of the model is governed by the distance between the original model parameter and its distorted counterpart and the sampling probability.

\textbf{Remark:}
The analysis of this theorem consists of two main steps. First, we present the bias-variance decomposition. Then, we provide bounds for bias and variance separately. The \textit{law of variance} is a generalized version of the \textit{sum-of-squares identity}. The total variation is decomposed as the summation of variation within treatments and the variation between treatments.






In the following lemma we decompose the utility of client $k$ as the summation of variance and the bias. Please refer to \pref{sec: variance_bias_decomp} for the full proof.

\begin{lem}[Bias-Variance Decomposition for Sum of Squares]\label{lem: variance_bias_decomposition}
   Let $W_{t}^{\calRO}$ represent the model parameter of client $k$ at round $t$. Then we have that
   \begin{align*}
    \text{GAP}(W_{t}^{\calRO}) & = \underbrace{\text{tr}(\text{Var}[W_{t}^{\calRO}])}_{\textbf{variance}} + \underbrace{\text{Bias}^2(W_{t}^{\calRO})}_{\textbf{bias}}.
\end{align*}
\end{lem}
\textbf{Remark:}
In this lemma we show that $\text{GAP}(W_{t}^{\calRO})$ with the sum-of-squares form could be decomposed as the summation of bias and variance. The bias of the original estimator $\text{Bias}(W_{t}^{\calRO})$ measures the gap of the utility using the true parameter and the estimated parameter (the bias of the original estimator is small is a basic requirement of the estimator).

The bias measures the gap of the utility using the true parameter and the estimated parameter. The bound for the bias gap is illustrated in the following lemma. Please refer to \pref{app: bias_gap_bound} for the full proof.
\begin{lem}\label{lem: bias_gap_bound_mt}
Let $W^*$ denote the optimal model parameter, i.e., $W^* = \arg\min_{W}\frac{1}{N} \sum_{i = 1}^N \calL(W, d_i)$, where $N$ represents the size of the mini-batch. Let $\text{Bias}(W_{t}^{\calRO}) = \|\E[W_{t}^{\calRO}] - W^*\|$. We have that 
\begin{align*}
    \bigg|\text{Bias}(\tildewt) - \text{Bias}(W_{t}^{\calRO})\bigg|&\le C_3\cdot {\text{TV}}(P_t^{(k)} || \wtilde P_t^{(k)} ),
\end{align*}
where $W_{t}^{\calRO} = W_{t-1}^{\calRO} - \frac{1}{|\calD^{(k)}|} \sum_{i = 1}^{|\calD^{(k)}|} \nabla\calL(W_{t-1}^{\calRO}, d_i^{(k)})$, and $\widetilde W_{t}^{(k)} = W_{t}^{(k)} + \delta_{t}^{(k)}$.
\end{lem}

The \textit{variance} represents the variation of the estimated values based on distinct datasets. The bound for the variance gap is illustrated in the following lemma. Please refer to \pref{app: variance_gap_bound_app} for the full proof.
\begin{lem}[Variance Gap]\label{lem: variance_gap_bound}
Let $N$ represent the size of the mini-batch. We have that
\begin{align}
    \text{Var}(\E[\tildewt|W_{t - 1}^{(k)}])  - \text{Var}(W_{t}^{\calRO})\le \underbrace{-\E(\text{Var}[\tildewt| W_{t - 1}^{(k)}])}_{\text{variance reduction}} + 2\sup\|W\|_2C_3\cdot {\text{TV}}(P^{(k)}_{t} || \wtilde P^{(k)}_{t} ).
\end{align}
\end{lem}





With the above lemmas, we are now ready to prove \pref{thm: upper_bound_of_epsilon_u}. Let $-\E(\text{Var}[\distpara])$ represent the variance reduction. This theorem provides an upper bound for utility loss, using the variance reduction and the total variation distance between the distorted distribution and the original distribution. From \pref{thm: upper_bound_of_epsilon_u}, we know that when $\E(\text{Var}[\distpara]) = C_3\cdot {\text{TV}}(P_t^{(k)} || \wtilde P_t^{(k)} )$, the utility loss is of client $k$ is $0$.

\subsection{Sampling Probability for Achieving Near-optimal Utility in Privacy-preserving Federated Learning}\label{sec: optimal_sample_prob}
Let $\xi = \max_{k\in [K]} \xi^{(k)}$, $\xi^{(k)} = \max_{w\in \mathcal{W}^{(k)}, d \in \mathcal{D}^{(k)}} \left|\log\left(\frac{f_{D^{(k)}|W^{(k)}}(d|w)}{f_{D^{(k)}}(d)}\right)\right|$ represent the maximum privacy leakage over all possible information $w$ released by client $k$, and $[K] = \{1,2,\cdots, K\}$. We define 
\begin{align}\label{eq: c_2_definition}
   C_2 = \frac{1}{2}(e^{2\xi}-1),
\end{align}
and
\begin{align}\label{eq: c_1_t_definition}
   C_{1,t} =\frac{1}{K}\sum_{k=1}^K \sqrt{{\text{JS}}(F^{(k)}_t || \wtilde F^{(k)}_t)}. 
\end{align}

The following lemma illustrates that the privacy leakage could be upper bounded using the total variation distance between $P^{\calRO}_t$ and $\wtilde P_t^{(k)}$.

\begin{lem}[Upper Bound for Privacy Leakage]\label{lem: privacy_leakage_upper_bound}
Let $F^{(k)}_t$ and $\wtilde F^{(k)}_t$ represent the belief of client $k$ about $S$ before and after observing the original parameter. Let $P^{\calRO}_t$ and $\wtilde P_t^{(k)}$ represent the distribution of the parameter of client $k$ at round $t$ before and after being protected. Assume that $C_2\cdot{\text{TV}}(P^{\calRO}_t || \wtilde P_t^{(k)})\le C_{1,t}$. The upper bound for the privacy leakage of client $k$ is
\begin{align*}
    \epsilon_{p,t}^{(k)} \le 2C_{1,t} - C_2\cdot {\text{TV}}(P^{\calRO}_t || \wtilde P_t^{(k)}),
\end{align*}
where $C_2$ is introduced in \pref{eq: c_2_definition}, and $C_{1,t}$ is introduced in \pref{eq: c_1_t_definition}. 
\end{lem}
\textbf{Remark:}
Intuitively, the privacy leakage decreases as the total variation distance increases, which is consistent with this upper bound. 

Given the requirement on privacy leakage, we can determine the least amount of distortion extent. Please refer to \pref{sec: from_privacy_to_distortion} for the full proof.
\begin{lem}\label{lem: from_privacy_to_distortion}
Let $C_{1,t} =\frac{1}{K}\sum_{k=1}^K \sqrt{{\text{JS}}(F^{(k)}_t || \wtilde F^{(k)}_t)}$. If the total variation distance is at least
\begin{align}
    \TV(P_t^{(k)}||\wtilde P_t^{(k)})\ge C_{1,t} - \tau_{p,t}^{(k)},
\end{align}
then the privacy leakage $\epsilon_{p,t}^{(k)}$ is at most $\tau_{p,t}^{(k)}$, where $C_{1,t}$ is introduced in \pref{eq: c_1_t_definition}.
\end{lem}
\textbf{Remark:}
The total variation distance between the distributions of the distorted model parameter $\wtilde P_t^{(k)}$ and that of the original model parameter $P_t^{(k)}$ serves as an upper bound for privacy leakage. With the requirement on the maximum amount of privacy leakage, we are now ready to derive a lower bound for the total variation distance.



The relationship between the total variation distance and the variance of the added noise is illustrated in the following lemma.
\begin{lem}[\cite{zhang2022no, zhang2022trading}]\label{lem: total_variation_and_variance}
    Let $\sigma^2$ represent the variance of the original model parameter, and $\sigma_\epsilon^2$ represent the variance of the added noise. Then 
    \begin{equation}
    \frac{1}{100}\min\left\{1, \frac{\sigma_\epsilon^2\sqrt{d}}{\sigma^2} \right\} \leq {\text{TV}}(P^{\calRO} || \wtilde P^{(k)} ) \leq  \frac{3}{2}\min\left\{1, \frac{\sigma_\epsilon^2\sqrt{d}}{\sigma^2} \right\},  
\end{equation}
where $d$ represents the number of dimension of the parameter.
\end{lem}

Please refer to \pref{sec: variance_and_privacy_leakage} for the full analysis. 
\begin{lem}\label{lem: variance_and_privacy_leakage}
Assume that $0< C_{1,t} - \tau_{p,t}^{(k)} < 0.01$, where $C_{1,t}$ is introduced in \pref{eq: c_1_t_definition}. Let $\sigma^2$ represent the variance of the original model parameter, and $\sigma_\epsilon^2$ represent the variance of the added noise. If the variance of the added noise $\sigma_\epsilon^2 = \frac{100\sigma^2(C_{1,t} - \tau_{p,t}^{(k)})}{\sqrt{d}}$, then the privacy leakage $\epsilon_{p,t}^{(k)}$ is at most $\tau_{p,t}^{(k)}$.
\end{lem}
\textbf{Remark: }
Given the variance of the added noise, we can guarantee that the lower bound of the total variation distance between the distributions of the distorted model parameter $\wtilde P_t^{(k)}$ and that of the original model parameter $P_t^{(k)}$ from \pref{lem: total_variation_and_variance}. Combined with \pref{lem: from_privacy_to_distortion}, it is guaranteed that the privacy leakage $\epsilon_{p,t}^{(k)}$ is at most $\tau_{p,t}^{(k)}$. 


The following lemma calculates the expectation of the model parameter $\distpara$. Please refer to \pref{sec: conditional_expectation_of_model_parameter} for the full proof.
\begin{lem}\label{lem: conditional_expectation_of_model_parameter}
Let $\distpara = W_{t-1}^{\calRO} - \frac{1}{N} \sum_{j = 1}^N \sum_{i = 1}^M \nabla\calL(W_{t-1}^{\calRO}, d_i^{(k)})\one\{d_i ^{(k)}\text{ is selected at } $j-$\text{th} \text{ round}\} + \delta_{t - 1}^{(k)}$, where $M$ represents the data size, and $N$ represents the total number of rounds for sampling. We have that
\begin{align}
    \E[\distpara] = W_{t-1}^{\calRO} - p\cdot\sum_{i = 1}^M \nabla\calL(W_{t-1}^{\calRO}, d_i^{(k)}) + \delta_{t - 1}^{(k)}.
\end{align}
\end{lem}
\textbf{Remark:} The expectation of the distorted model parameter is related to the sampling probability $p$ and the added noise $\delta_{t - 1}^{(k)}$.

With the expectation of the distorted model parameter, the following theorem further calculates the variance of the distorted model parameter $\distpara$. Fixing $W_{t-1}^{\calRO}$ and data $d_i$, then $\text{Var}[\distpara]$ depends on $p$. Please refer to \pref{app: variance_of_distorted_model_parameter} for the full proof.
\begin{thm}\label{thm: variance_of_distorted_model_parameter}
We denote $p$ as the sampling probability. That is, each data of each client is sampled with probability $p$ to generate the batch. Let $\distpara = W_{t-1}^{\calRO} - \frac{1}{N} \sum_{i = 1}^N \nabla\calL(W_{t-1}^{\calRO}, d_i^{(k)}) + \delta_{t - 1}^{(k)}$, where $M$ represents the data size, and $N$ represents the total number of rounds for sampling. We have that
\begin{align}
    \text{Var}[\distpara] = p\cdot (1-p)\cdot\sum_{i = 1}^M \left(\nabla\calL(W_{t-1}^{\calRO}, d_i^{(k)})\right)^2.
\end{align}
\end{thm}
\textbf{Remark:} The variance of the distorted model parameter is related to the sampling probability and the gradient of the dataset.

With the following theorem, we can find the optimal sampling probability for achieving near-optimal utility, and meanwhile satisfies the requirement on privacy.
\begin{thm}\label{thm: from _utility_to_probability} 
   Let \pref{assump: upper_bounded_by_c_4} hold. Given the requirement that the privacy leakage $\epsilon_{p,t}^{(k)}$ should not exceed $\tau_{p,t}^{(k)}$. If the sampling probability $p$ satisfies
\begin{align}\label{eq: sample_prob_equation}
    p(1-p)\ge \frac{C_6\cdot(C_{1,t} - \tau_{p,t}^{(k)})}{\sum_{i = 1}^M \left(\nabla\calL(W_{t-1}^{\calRO}, d_i)\right)^2},
\end{align}  
   then client $k$ achieves near-optimal utility, where $C_{1,t}$ is introduced in \pref{eq: c_1_t_definition}, and $C_4$ is introduced in \pref{assump: upper_bounded_by_c_4}.
\end{thm}
\textbf{Remark:}
Let $-\E(\text{Var}[\distpara])$ represent the variance reduction. \pref{thm: upper_bound_of_epsilon_u} provides an upper bound for utility loss, using the variance reduction and the total variation distance between the distorted distribution and the original distribution. From \pref{thm: upper_bound_of_epsilon_u}, we know that when $\E(\text{Var}[\distpara]) = C_6\cdot {\text{TV}}(P_t^{(k)} || \wtilde P_t^{(k)} )$, the utility loss is of client $k$ is $0$. \pref{thm: variance_of_distorted_model_parameter} illustrates that the variance of the distorted model parameter is related to the sampling probability and the gradient of the dataset. \pref{lem: from_privacy_to_distortion} provides a lower bound for total variation distance. Therefore, when the sampling probability $p$ satisfies \pref{eq: sample_prob_equation}, then client $k$ achieves near-optimal utility.


\subsection{Optimal Trade-off Between Utility Loss and Privacy Leakage}
In this section, we derive the optimal trade-off between utility loss and privacy leakage. Please refer to \pref{sec: bounds_for_trade_off} for the full proof.

The following theorem provides an upper bound for utility loss of client $k$ at round $t$.
\begin{thm}[Upper Bound for Trade-off]\label{thm: upper_bound_trade_off_mt}
  Let \pref{assump: upper_bounded_by_c_3} and \pref{assump: upper_bounded_by_c_4} hold. We have that
  \begin{align*}
     \epsilon_{p,t}^{(k)} + \frac{C_2}{C_6}\cdot\epsilon_{u,t}^{(k)} &\le -\frac{C_2}{C_6}\cdot\E(\text{Var}[\distpara| W_{t - 1}^{(k)}]) + 2 C_{1,t}^{(k)}.
\end{align*}
where $C_{1,t}^{(k)} = \sqrt{{\text{JS}}(F^{(k)}_t || \wtilde F^{(k)}_t)}$, $C_2$ is introduced in \pref{eq: c_2_definition}, and $C_6$ is introduced in \pref{thm: upper_bound_of_epsilon_u}.
\end{thm}
\textbf{Remark:}
\pref{thm: upper_bound_of_epsilon_u} illustrates the upper bound of utility loss using variance reduction and the total variation distance. \pref{lem: privacy_leakage_upper_bound} presents the relationship between the total variation distance and the privacy leakage. Combining \pref{thm: upper_bound_of_epsilon_u} and \pref{lem: privacy_leakage_upper_bound}, we can express the upper bound of utility loss using privacy leakage.

The following theorem provides a lower bound for trade-off between privacy and utility. 
\begin{thm}[Lower Bound for Trade-off, see Theorem 4.1 of \cite{zhang2022no}]\label{lem: utility-privacy trade-off_JSD_mt} 
Let $\epsilon_{p,t}^{(k)}$ be defined in \pref{defi: average_privacy_JSD}, and let $\epsilon_{u,t}^{(k)}$ be defined in \pref{defi: utility_loss}, with \pref{assump: assump_of_Delta} we have:
\begin{align}\label{eq: total_variation-privacy trade-off_app_2}
 \epsilon_{p,t}^{(k)} + C_d\cdot \epsilon_{u,t}^{(k)}\ge C_{1,t}^{(k)},
\end{align}
where $C_{1,t}^{(k)} = \sqrt{{\text{JS}}(F^{(k)}_t || \wtilde F^{(k)}_t)}$, $C_d = \frac{\gamma}{4\Delta}(e^{2\xi}-1)$, where $\xi^{(k)}$=$\max_{w\in \mathcal{W}^{(k)}, d \in \mathcal{D}^{(k)}} \left|\log\left(\frac{f_{D^{(k)}|W^{(k)}}(d|w)}{f_{D^{(k)}}(d)}\right)\right|$, $\xi$=$\max_{k\in [K]} \xi^{(k)}$ represents the maximum privacy leakage over all possible information $w$ released by client $k$, and $\Delta$ is introduced in \pref{assump: assump_of_Delta}.
\end{thm}
\textbf{Remark:}
This theorem states that the summation of privacy leakage and utility loss against the semi-honest attacker is constrained by a constant. The utility of the model may be diminished if privacy protection is strengthened, and vice versa. Notice that $\gamma = \frac{\sum_{k=1}^K {\text{TV}}(P^{\calRO} || \wtilde P^{(k)})}{{\text{TV}}(P_a || \wtilde P_a )}$. From Lemma C.2 of \cite{zhang2022no}, $\frac{1}{150}\le\gamma\le 150$.

With the upper bound and lower bound for trade off between privacy and utility, we are ready to derive the condition for achieving optimal trade-off, which is illustrated in the following theorem.
\begin{thm}[Optimal Trade-off]\label{thm: optimal_trade_off}
Consider the scenario where $C_d = \frac{C_2}{C_6}$. If $C_{1,t}^{(k)} = \frac{C_2}{C_6}\cdot\E(\text{Var}[\distpara| W_{t - 1}^{(k)}])$, then the optimal trade-off is achieved, where $C_{1,t}^{(k)} = \sqrt{{\text{JS}}(F^{(k)}_t || \wtilde F^{(k)}_t)}$, $C_d = \frac{\gamma}{4\Delta}(e^{2\xi}-1)$, where $\xi^{(k)}$=$\max_{w\in \mathcal{W}^{(k)}, d \in \mathcal{D}^{(k)}} \left|\log\left(\frac{f_{D^{(k)}|W^{(k)}}(d|w)}{f_{D^{(k)}}(d)}\right)\right|$, $\xi$=$\max_{k\in [K]} \xi^{(k)}$ represents the maximum privacy leakage over all possible information $w$ released by client $k$, $\Delta$ is introduced in \pref{assump: assump_of_Delta}, and $C_6$ is introduced in \pref{thm: upper_bound_of_epsilon_u}.
\end{thm}

\section{HFL Algorithms with near-optimal Utility}




Our goal is to design a sampling strategy that satisfies the privacy constraint, and at the same time achieving near-optimal utility. Given the privacy budget $\tau_{p,t}^{(k)}$ for client $k$ at round $t$, the total variation distance between two distributions is then calculated via \pref{lem: from_privacy_to_distortion}. We use the randomization mechanism as an illustrative example, which adds a random noise following the normal distribution on the transmitted model parameter. The subroutine \mathsc{DistortModelParameter} adds noise according to the calculated variance for randomization mechanism. The variance of the added noise is further derived according to \pref{lem: total_variation_and_variance}, which guarantees the privacy constraint is satisfied. With the calculated total variation distance and the theoretical result illustrated in \pref{thm: upper_bound_of_epsilon_u}, the sampling probability for achieving near-optimal utility is then calculated via \pref{eq: sample_prob_equation} (\pref{thm: from _utility_to_probability} in \pref{sec: optimal_sample_prob}). With the calculated sampling probability, the client constructs a mini-batch $\calS^{(k)}$ from his dataset $\calD^{(k)}$. The client then updates his model parameter with the mini-batch $\calS^{(k)}$. These observations lead to our algorithm that achieves near-optimal utility and simultaneously satisfies the requirements that the privacy leakage of client $k$ at round $t$ does not exceed $\tau_{p,t}^{(k)}$.




\begin{algorithm}[!htp]
    \caption{\mathsc{FLwithnear-optimalUtility}}
    \begin{algorithmic}
     \State \textbf{Initialization:} $\tau_{p,t}^{(k)}$, privacy budget of client $k$ at round $t$; $C_{1,t}$, a problem-dependent constant introduced in \pref{eq: c_1_t_definition}. 
    \State $T$: the number of training steps for the model parameter; 
     \State $W_0 = \wtilde W_0$: model parameter initialized by the server
    \For{$t=0, 1, \ldots, T$}
      \For{each client $k\in [K]$}
         \State $\text{var}_t^{(k)}\leftarrow\frac{100\sigma^2(C_{1,t} - \tau_{p,t}^{(k)})}{\sqrt{d}}$.
         \State $W_{t}^{\calRO}\leftarrow \mathsc{ClientModelTraining}(k, \wtilde W_t, \tau_{p,t}^{(k)}).$
         \State $\wtilde W_{t+1}^{(k)}\leftarrow\mathsc{DistortModelParameter}(W_{t}^{\calRO}, \text{var}^{(k)}).$
        \EndFor       
        \State \textbf{Server execute:}
        \State $\wtilde W_{t+1}\leftarrow \sum_{k = 1}^K \frac{n_k}{n}\wtilde W_{t+1}^{(k)}$.
    \EndFor 
    \end{algorithmic}\label{alg: adaptive learning algorithm}
\end{algorithm}

\pref{thm: upper_bound_of_epsilon_u} states that $\epsilon_{u,t}^{(k)} \le -\E(\text{Var}[\tildewt| W_{t - 1}^{(k)}]) + C_6\cdot {\text{TV}}(P_t^{(k)} || \wtilde P_t^{(k)} )$. When the sampling probability $p$ satisfies that
\begin{align}\label{eq: sampling_probability}
    p\cdot (1-p)\cdot\sum_{i = 1}^M \left(\nabla\calL(W_{t-1}^{\calRO}, d_i)\right)^2\le\tau_{p,t}^{(k)},
\end{align}
the utility loss is $0$, and meanwhile the privacy leakage is at most $\tau_{p,t}^{(k)}$.
We can near-optimal utility via adjusting the sampling probability $p$ for constructing the mini-batch according to \pref{eq: sampling_probability}. The subroutine \mathsc{ClientModelTraining} updates the model parameters locally using the private data of client $k$.

\begin{algorithm}[!htp]
    \caption{\mathsc{ClientModelTraining}($k, \wtilde W_t, \tau_{p,t}^{(k)}$)}
    \label{alg: clientTraining}
      \begin{algorithmic}
      \State Given $\tau_{p,t}^{(k)}$, $p$ is calculated according to \pref{eq: sample_prob_equation}.
      \State Sample a dataset $S^{(k)}$ from $\mathcal{D}^{(k)} = \{d_{1}^{(k)},\cdots, d_{|\mathcal{D}^{(k)}|}^{(k)}\}$ with probability $p$.
       \State $W_{t}^{\calRO} \leftarrow \wtilde W_t -\eta\cdot\frac{1}{|S^{(k)}|}\sum_{i\in S^{(k)}}\nabla \calL(\wtilde W_t, d_i^{(k)})$.
       \end{algorithmic}
\end{algorithm} 



\begin{algorithm}[!htp]
    \caption{\mathsc{DistortModelParameter} ($W_{t}^{(k)}, \sigma^{2}$)}
    \label{alg: LearnDistortionExtent}
    \begin{algorithmic}
   \State $\epsilon_t^{(k)}\sim \calN(0, \sigma^{2})$.
    \State $W_{t+1}^{(k)}\leftarrow W_{t}^{(k)} + \epsilon_t^{(k)}$.
    \State return $W_{t+1}^{(k)}$.
    \end{algorithmic}
\end{algorithm}
\section{Conclusion and Future Works}
We measure the utility via the gap between the original model parameter and the distorted model parameter, and provide an upper bound for utility loss via bias-variance decomposition. Based on this upper bound, we provide an algorithm that achieves near-optimal utility, and meanwhile satisfies the requirement on privacy leakage. The main techniques of the proposed protection mechanism are parameter distortion and data generation, which are generic and have a wide range of applications. Furthermore, we derive an upper bound for the trade-off between privacy and utility, which when combined with the lower bound shown in NFL, creates the prerequisites for obtaining the best possible trade-off.

\textbf{ACKNOWLEDGMENTS}
We thank Lixin Fan for helpful discussion. This work was partially supported by the National Key Research and Development Program of China under Grant 2020YFB1805501 and Hong Kong RGC TRS T41-603/20-R.




\bibliography{main}
\bibliographystyle{ACM-Reference-Format}

\clearpage

\newpage
\onecolumn
\appendix
\section{Notation Table}

\begin{table*}[!htp]
\footnotesize
  \centering
  \setlength{\belowcaptionskip}{15pt}
  \caption{Table of Notation}
  \label{table: notation}
    \begin{tabular}{cc}
    \toprule
    Notation & Meaning\cr
    \midrule\
    $\epsilon_{p,t}^{(k)}$ & Privacy leakage (Def. \ref{defi: average_privacy_JSD})\cr
    $\epsilon_{u,t}^{(k)}$ & Utility loss (Def. \ref{defi: utility_loss}) \cr
    $D$ & Private information, including private data and statistical information\cr
    $W_{a}$ & parameter for the federated model\cr
    $W^{\calRO}$ & Unprotected model information of client $k$\cr
    $\wtilde W^{(k)}$ & Protected model information of client $k$\cr
 $P^{(k)}$ & Distribution of unprotected model information of client $k$\cr
 $\wtilde P^{(k)}$ & Distribution of protected model information of client $k$\cr
 $\calW^{(k)}$ & Union of the supports of $P^{(k)}$ and $\wtilde P^{(k)}$\cr
 $\what F^{(k)}$ & Adversary's prior belief distribution about the private information of client $k$\cr
 $\wtilde F^{(k)}$ & Adversary's belief distribution about client $k$ after observing the protected information\cr
 $F^{\calRO}$ & Adversary's belief distribution about client $k$ after observing the unprotected information\cr
 $\text{JS}(\cdot||\cdot)$ & Jensen-Shannon divergence between two distributions\cr
 $\text{TV}(\cdot||\cdot)$ & Total variation distance between two distributions\cr
    \bottomrule
    \end{tabular}
\end{table*}

\section{Bounds for Privacy Leakage}
In this section, we provide lower and upper bounds for privacy leakage. 

\subsection{Lower Bound for Privacy Leakage}

\cite{zhang2022no} illustrated that the privacy leakage could be lower bounded by the total variation distance between $P^{\calRO}_t$ and $\wtilde P^{(k)}_t$, as is shown in the following lemma.

\begin{lem}[\cite{zhang2022no}]\label{lem: total_variation-privacy trade-off_app_1}
Let $\epsilon_{p,t}^{(k)}$ be introduced in \pref{defi: average_privacy_JSD}. Let $P^{\calRO}_t$ and $\wtilde P^{(k)}_t$ represent the distribution of the parameter of client $k$ before and after being protected. Let $F^{(k)}_t$ and $\wtilde F^{(k)}_t$ represent the belief of client $k$ about $D$ before and after observing the original parameter. Then we have

\begin{align*}
\epsilon_{p,t}^{(k)} \ge \frac{1}{K}\sum_{k=1}^K \sqrt{{\text{JS}}(\wtilde F^{(k)}_t || F^{(k)}_t)} - \frac{1}{K}\sum_{k=1}^K \frac{1}{2}(e^{2\xi}-1)\cdot {\text{TV}}(\wtilde P^{(k)}_t || P^{\calRO}_t).
\end{align*}
\end{lem}
\subsection{Upper Bound for Privacy Leakage}
In this section, we provide an upper bound for privacy leakage using Wesserstein distance, which is derived as follows.

\begin{lem}[\cite{duchi2013local}]\label{lem: log_upper_bound}
For two positive numbers $a$ and $b$, we have that
\begin{align*}
 \left|\log\left(\frac{a}{b}\right)\right| \le \frac{|a-b|}{\min\{a,b\}}.
\end{align*}
\end{lem}

\begin{lem}\label{lem: JSBound}
Let $P^{(k)\calRRO}$ and $\wtilde P^{(k)}$ represent the distribution of the parameter of client $k$ before and after being protected. Let $\wtilde F^{(k)}$ and $F^{(k)}$ represent the belief of client $k$ about $D$ after observing the protected and original parameter. Then we have
\begin{align*}
{\text{JS}}(\wtilde F^{(k)} || F^{(k)})\le \frac{1}{4}(e^{2\xi}-1)^2{\text{TV}}(\wtilde P^{(k)} || P^{(k)\calRRO})^2. 
\end{align*}
\end{lem}

\begin{proof}

Let $\overline F^{(k)} = \frac{1}{2}(\wtilde F^{(k)}+ F^{(k)})$. We have

\begin{align*}
{\text{JS}}(\wtilde F^{(k)} || F^{(k)}) & = \frac{1}{2}\left[KL\left(\wtilde F^{(k)} || \overline F^{(k)}\right) + KL\left(F^{(k)} || \overline F^{(k)}\right)\right]\\
& = \frac{1}{2}\left[\int_{\mathcal{D}^{(k)}} \wtilde f_{D^{(k)}}(d)\log\frac{\wtilde f_{D^{(k)}}(d)}{\overline f_{D^{(k)}}(d)}\textbf{d}\mu(d) + \int_{\mathcal{D}^{(k)}} f_{D^{(k)}}(d)\log\frac{f_{D^{(k)}}(d)}{\overline f_{D^{(k)}}(d)}\textbf{d}\mu(d)\right]\\
& = \frac{1}{2}\left[\int_{\mathcal{D}^{(k)}} \wtilde f_{D^{(k)}}(d)\log\frac{\wtilde f_{D^{(k)}}(d)}{\overline f_{D^{(k)}}(d)}\textbf{d}\mu(d) - \int_{\mathcal{D}^{(k)}} f_{D^{(k)}}(d)\log\frac{\overline f_{D^{(k)}}(d)}{f_{D^{(k)}}(d)}\textbf{d}\mu(d)\right]\\
&\le \frac{1}{2}\int_{\mathcal{D}^{(k)}}\left|\wtilde f_{D^{(k)}}(d) - f_{D^{(k)}}(d)\right|\left|\log\frac{\overline f_{D^{(k)}}(d)}{f_{D^{(k)}}(d)}\right|\textbf{d}\mu(d),
\end{align*}
where the inequality is due to $\frac{\wtilde f_{D^{(k)}}(d)}{\overline f_{D^{(k)}}(d)}\le \frac{\overline f_{D^{(k)}}(d)}{f_{D^{(k)}}(d)}$.

\textbf{Bounding $\left|\wtilde f_{D^{(k)}}(d) - f_{D^{(k)}}(d)\right|$.}

Let $\mathcal U^{(k)} = \{w\in\mathcal W^{(k)}: d\wtilde P^{(k)}(w) - dP^{(k)\calRRO}(w)\ge 0\}$, and $\mathcal V^{(k)} = \{w\in\mathcal W^{(k)}: d\wtilde P^{(k)}(w) - dP^{(k)\calRRO}(w)< 0\}$.

Then we have 

\begin{align}\label{eq:initial_step_{JS}}
    &\left|\wtilde f_{D^{(k)}}(d) - f_{D^{(k)}}(d)\right| 
    = \left|\int_{\mathcal W^{(k)}} f_{D^{(k)}|W^{(k)}}(d|w)[d \wtilde P^{(k)}(w) - d P^{(k)\calRRO}(w)]\right|\nonumber\\
    &= \left|\int_{\mathcal{U}^{(k)}} f_{D^{(k)}|W^{(k)}}(d|w)[d \wtilde P^{(k)}(w) - d P^{(k)\calRRO}(w)] + \int_{\mathcal{V}^{(k)}} f_{D^{(k)}|W^{(k)}}(d|w)[d \wtilde P^{(k)}(w) - d P^{(k)\calRRO}(w)]\right|\nonumber\\
    &\le\left(\sup_{w\in\mathcal{W}^{(k)}} f_{D^{(k)}|W^{(k)}}(d|w) - \inf_{w\in\mathcal{W}^{(k)}} f_{D^{(k)}|W^{(k)}}(d|w)\right)\int_{\mathcal{U}^{(k)}} [d \wtilde P^{(k)}(w) - d P^{(k)\calRRO}(w)].
\end{align}



Notice that

\begin{align*}
    \sup_{w\in\mathcal W^{(k)}} f_{D^{(k)}|W^{(k)}}(d|w) - \inf_{w\in\mathcal W^{(k)}} f_{D^{(k)}|W^{(k)}}(d|w) = \inf_{w\in\mathcal W^{(k)}} f_{D^{(k)}|W^{(k)}}(d|w)\left|\frac{\sup_{w\in\mathcal W^{(k)}} f_{D^{(k)}|W^{(k)}}(d|w)}{\inf_{w\in\mathcal W^{(k)}} f_{D^{(k)}|W^{(k)}}(d|w)}-1\right|.
\end{align*}

From the definition of $\xi$, we know that for any $w\in\mathcal W^{(k)}$,
\begin{align*}
    e^{-\xi}\le\frac{f_{D^{(k)}|W^{(k)}}(d|w)}{f_{D^{(k)}}(d)}\le e^{\xi},
\end{align*}

Therefore, for any pair of parameters $w,w'\in\mathcal W^{(k)}$, we have
\begin{align*}
    \frac{f_{D^{(k)}|W^{(k)}}(d|w)}{f_{D^{(k)}|W^{(k)}}(d|w')} = \frac{f_{D^{(k)}|W^{(k)}}(d|w)}{f_{D^{(k)}}(d)}/\frac{f_{D^{(k)}|W^{(k)}}(d|w')}{f_{D^{(k)}}(d)}\le e^{2\xi}. 
\end{align*}

Therefore, the first term of \pref{eq:initial_step_{JS}} is bounded by

\begin{align}\label{eq: bound_1_term_1_{JS}_ratio}
    \sup_{w\in\mathcal W^{(k)}} f_{D^{(k)}|W^{(k)}}(d|w) - \inf_{w\in\mathcal W^{(k)}} f_{D^{(k)}|W^{(k)}}(d|w) \le \inf_{w\in\mathcal W^{(k)}} f_{D^{(k)}|W^{(k)}}(d|w)(e^{2\xi}-1).
\end{align}

From the definition of total variation distance, we have
\begin{align}\label{eq: bound_1_term_2_{JS}_ratio}
    \int_{\UU} [d \wtilde P^{(k)}(w) - d P^{(k)\calRRO}(w)] = {\text{TV}}(P^{(k)\calRRO} || \wtilde P^{(k)}).
\end{align}

Combining \pref{eq: bound_1_term_1_{JS}_ratio} and \pref{eq: bound_1_term_2_{JS}_ratio}, we have
\begin{align}\label{eq: bound_for_the_gap}
        |\wtilde f_{D^{(k)}}(d) - f_{D^{(k)}}(d)| &=\left(\sup_{w\in\mathcal W^{(k)}} f_{D^{(k)}|W^{(k)}}(d|w) - \inf_{w\in\mathcal W^{(k)}} f_{D^{(k)}|W^{(k)}}(d|w)\right)\int_\UU [d \wtilde P^{(k)}(w) - d P^{(k)\calRRO}(w)]\nonumber\\
        &\le\inf_{w\in\mathcal W^{(k)}} f_{D^{(k)}|W^{(k)}}(d|w)(e^{2\xi}-1){\text{TV}}(P^{(k)\calRRO} || \wtilde P^{(k)}).
\end{align}


\textbf{Bounding $\left|\log\left(\frac{\overline f_{D^{(k)}}(d)}{ f_{D^{(k)}}(d)}\right)\right|.$}

We have that

\begin{align}\label{eq: bound_for_log_ratio}
    \left|\log\frac{\overline f_{D^{(k)}}(d)}{f_{D^{(k)}}(d)}\right|&\le\frac{|\overline f_{D^{(k)}}(d) - f_{D^{(k)}}(d)|}{\min\{\overline f_{D^{(k)}}(d), f_{D^{(k)}}(d)\}}\nonumber\\
    &=\frac{|\wtilde f_{D^{(k)}}(d) - f_{D^{(k)}}(d)|}{2\min\{\overline f_{D^{(k)}}(d), f_{D^{(k)}}(d)\}}\nonumber\\
    &\le \frac{\inf_{w\in\mathcal W^{(k)}} f_{D^{(k)}|W^{(k)}}(d|w)(e^{2\xi}-1){\text{TV}}(P^{(k)\calRRO} || \wtilde P^{(k)})}{2\min\{\overline f_{D^{(k)}}(d), f_{D^{(k)}}(d)\}}\nonumber\\
    &\le \frac{1}{2}(e^{2\xi}-1){\text{TV}}(P^{(k)\calRRO} || \wtilde P^{(k)}),
\end{align}
where the first inequality is due to \pref{lem: log_upper_bound}, the third inequality is due to $\min\{\overline f_{D^{(k)}}(d), f_{D^{(k)}}(d)\}\ge \min\{\wtilde f_{D^{(k)}}(d), f_{D^{(k)}}(d)\}\ge \inf\limits_{\small{w\in\mathcal W^{(k)}}} f_{D^{(k)}|W^{(k)}}(d|w)$.

Combining \pref{eq: bound_for_the_gap} and \pref{eq: bound_for_log_ratio}, we have
\begin{align*}
    {\text{JS}}(\wtilde F^{(k)} || F^{(k)}) & \le  \frac{1}{2}\left[\int_{\mathcal{D}^{(k)}} \left|(\wtilde f_{D^{(k)}}(d) - f_{D^{(k)}}(d))\right| \left|\log\frac{\overline f_{D^{(k)}}(d)}{f_{D^{(k)}}(d)}\right|\textbf{d}\mu(d)\right]\\
    &\le \frac{1}{4}(e^{2\xi}-1)^2{\text{TV}}(P^{(k)\calRRO} || \wtilde P^{(k)})^2\int_{\mathcal{D}^{(k)}} \inf_{w\in\mathcal W^{\calRRO(k)}} f_{D^{(k)}|W^{(k)}}(d|w)\textbf{d}\mu(d)\\
    &\le\frac{1}{4}(e^{2\xi}-1)^2{\text{TV}}(P^{(k)\calRRO} || \wtilde P^{(k)})^2.
\end{align*}

\end{proof}
\section{Assumption of previous work}
\begin{definition}[Optimal parameters]
Let $\mathcal W^{*}_a$ represent the set of parameters achieving the maximum utility. Specifically, 
\begin{align*}
    \mathcal W^{*}_a = \argmax_{w\in\mathcal W_a}\frac{1}{K}\sum_{k=1}^K U^{(k)}(w),
\end{align*}
where $U^{(k)}(w) = \mathbb{E}_{D^{(k)}}\frac{1}{|D^{(k)}|}\sum_{d\in D^{(k)}}U(w,d)$ is the expected utility taken over $D^{(k)}$ sampled from distribution $P^{(k)}$.
\end{definition}





\begin{definition}[Near-optimal parameters]\label{defi: neighbor_set}
Let $\wtilde\calW_a$ represent the support of the protected distribution of the aggregated model information. Given a non-negative constant $c$, the \textit{near-optimal parameters} is defined as
$$\calW_{c} = \left\{w\in\wtilde\calW_a: \left|\frac{1}{K}\sum_{k=1}^K  U^{(k)}(w^{*})-\frac{1}{K}\sum_{k=1}^K U^{(k)}(w)\right|\le c, \forall w^{*}\in\mathcal W^{*}_a\right\}.$$


\end{definition}

\begin{assumption}\label{assump: assump_of_Delta}
Let $\Delta$ be the \textit{maximum} constant that satisfies
\begin{align}
     \int_{\wtilde{\mathcal W}_a} \wtilde p_{W_a}(w)\one\{w\in\calW_{\Delta}\} dw\le\frac{{\text{TV}}(P_a || \wtilde P_a )}{2},
\end{align}
where $\wtilde p$ represents the probability density function of the protected model information. We assume that $\Delta$ is positive, i.e., $\Delta >0$.

\end{assumption}
\textbf{Remark:}\\
(1) This assumption implies that the cumulative density of the near-optimal parameters as defined in Def. \ref{defi: neighbor_set} is bounded. This assumption excludes the cases where the utility function is constant or indistinguishable between the optimal parameters and a certain fraction of parameters.  \\
(2) Note that $\Delta^{(k)}$ is independent of the threat model of the adversary and $\Delta^{(k)}$ is a constant when the protection mechanism, the utility function, and the data sets are fixed.\\


First, we present the bias-variance decomposition. Then, we provide bounds for bias and variance separately.

\section{Analysis for Lemma \ref{lem: variance_bias_decomposition}}\label{sec: variance_bias_decomp}

In the following lemma we show that $\text{GAP}(W_{t}^{\calRO})$ with the sum-of-squares form could be decomposed as the summation of both bias and variance. The bias of the original estimator $\text{Bias}(W_{t}^{\calRO})$ measures the gap of the utility using the true parameter and the estimated parameter (the bias of the original estimator is small is a basic requirement of the estimator).

\begin{lem}[Variance-Bias Decomposition for Sum of Squares]\label{lem: variance_bias_decomp}
   Let $W_{t}^{\calRO}$ represent the model parameter of client $k$ at round $t$. Then we have that
   \begin{align*}
    \text{GAP}(W_{t}^{\calRO}) & = \underbrace{\text{tr}(\text{Var}[W_{t}^{\calRO}])}_{\textbf{variance}} + \underbrace{\text{Bias}^2(W_{t}^{\calRO})}_{\textbf{bias}}.
\end{align*}
\end{lem}
\begin{proof}

Let $W_{t}^{\calRO}\in\R^d$ represent the model parameter at round $t$, which is updated using the mini-batch from client $k$, and $W^*\in\R^d$ denote the optimal model parameter.



Then we have that
\begin{align*}
    \text{GAP}(W_{t}^{\calRO}) & =  \E\|W_{t}^{\calRO} - W^*\|^2\\
    & =  \E\|W_{t}^{\calRO} - \E[W_{t}^{\calRO}] + \E[W_{t}^{\calRO}] - W^*\|^2\\
    & = \underbrace{ \E[\|W_{t}^{\calRO} - \E[W_{t}^{\calRO}]\|^2] }_{\textbf{variance}} + \underbrace{ \E[\|\E[W_{t}^{\calRO}] - W^*\|^2]}_{\textbf{bias}}\\
    & = \text{Var}[W_{t}^{\calRO}] +  \E[\text{Bias}^2(W_{t}^{\calRO})]\\
    & = \underbrace{\text{tr}(\text{Var}[W_{t}^{\calRO}])}_{\textbf{variance}} + \underbrace{\E[\text{Bias}^2(W_{t}^{\calRO})]}_{\textbf{bias}},
\end{align*}
where the last equation is due to $\text{Var}[W_{t}^{\calRO}] = \text{tr}(\text{Var}[W_{t}^{\calRO}])$, and we denote $\text{Bias}(W_{t}^{\calRO}) = \|\E[W_{t}^{\calRO}] - W^*\|$. 
Notice that the expectation is taken over the randomness of $W_{t}^{\calRO}$.
\end{proof}

\section{Analysis for Lemma \ref{lem: bias_gap_bound_mt}}\label{app: bias_gap_bound}
Let $P^{(k)}$ represent the distribution of $W_{t}^{\calRO}$, and $\wtilde P^{(k)}$ represent the distribution of $\tildewt$. Now we provide an upper bound for the gap using total variation distance. The total variation distance measures the distance between the distributions of the distorted parameter and the true parameter. Recall that $W_{t}^{\calRO} = W_{t-1}^{\calRO} - \frac{1}{|\calD^{(k)}|} \sum_{i = 1}^{\calD^{(k)}} \nabla\calL(W_{t-1}^{\calRO}, d_i)$.

\begin{lem}\label{lem: bound_using_dtv_norm_of_para_app}
We define $C(W) = \|W\|$. Let \pref{assump: upper_bounded_by_c_3} hold. That is, $C(W)\in [0, C_3]$ for any $W \in \mathcal{W}^{(k)}$. We have that
\begin{align*}
    \left|\E[C(\distpara)] - \E[C(W_{t}^{\calRO})]\right|\le C_3\cdot\TV(\wtilde P^{(k)}_{t} || P^{(k)}_{t}),
\end{align*}
where the expectation is taken over the randomness of $W_{t}^{\calRO}$ and the randomness of distortion.
\end{lem}
\begin{proof}
Let $\mathcal U^{(k)} = \{W\in\mathcal W^{(k)}: d\wtilde P^{(k)}_{t}(W) - dP^{(k)}_{t}(W)\ge 0\}$, and $\mathcal V^{(k)} = \{W\in\mathcal W^{(k)}: d\wtilde P^{(k)}_{t}(W) - dP^{(k)}_{t}(W)< 0\}$. Then we have
\begin{align*}  
    & \left[\mathbb E_{W\sim P^{(k)}_{t}}[C(W)] - \mathbb E_{W\sim \wtilde P^{(k)}_{t}}[C(W)]\right]\\
    & = \left[\int_{\mathcal W^{(k)}} C(W) dP^{(k)}_{t}(W) - \int_{\mathcal W^{(k)}} C(W) d\wtilde P^{(k)}_{t}(W)\right]\\
    & = \left[\int_{\mathcal{V}^{(k)}} C(W)[d P^{(k)}_{t}(W) - d \wtilde P^{(k)}_{t}(W)] - \int_{\mathcal{U}^{(k)}} C(W)[d \wtilde P^{(k)}_{t}(W) - d P^{(k)}_{t}(W)]\right]\\
    &\le \frac{C_3}{K}\sum_{k=1}^K\int_{\mathcal{V}^{(k)}} [d P^{(k)}_{t}(W) - d \wtilde P^{(k)}_{t}(W)]\\
    &=C_3\cdot {\text{TV}}(P^{(k)}_{t} || \wtilde P^{(k)}_{t} ).
\end{align*}
\end{proof}

\begin{lem}\label{lem: bound_using_dtv_app}
We define $C(W) = \|\E[W] - W^*\|$. Let \pref{assump: upper_bounded_by_c_4} hold. That is, $C(W)\in [0, C_4]$ for any $W \in \mathcal{W}^{(k)}$. We have that
\begin{align*}
    \left|\E[C(\distpara)] - \E[C(W_{t}^{\calRO})]\right|\le C_4\cdot\TV(\wtilde P^{(k)}_{t} || P^{(k)}_{t}),
\end{align*}
where the expectation is taken over the randomness of $W_{t}^{\calRO}$ and the randomness of the distortion.

\end{lem}
\begin{proof}
Let $\mathcal U^{(k)} = \{W\in\mathcal W^{(k)}: d\wtilde P^{(k)}_{t}(W) - dP^{(k)}_{t}(W)\ge 0\}$, and $\mathcal V^{(k)} = \{W\in\mathcal W^{(k)}: d\wtilde P^{(k)}_{t}(W) - dP^{(k)}_{t}(W)< 0\}$. Then we have
\begin{align*}  
    & \left[\mathbb E_{W\sim P^{(k)}_{t}}[C(W)] - \mathbb E_{W\sim \wtilde P^{(k)}_{t}}[C(W)]\right]\\
    & = \left[\int_{\mathcal W^{(k)}} C(W) dP^{(k)}_{t}(W) - \int_{\mathcal W^{(k)}} C(W) d\wtilde P^{(k)}_{t}(W)\right]\\
    & = \left[\int_{\mathcal{V}^{(k)}} C(W)[d P^{(k)}_{t}(W) - d \wtilde P^{(k)}_{t}(W)] - \int_{\mathcal{U}^{(k)}} C(W)[d \wtilde P^{(k)}_{t}(W) - d P^{(k)}_{t}(W)]\right]\\
    &\le \frac{C_4}{K}\sum_{k=1}^K\int_{\mathcal{V}^{(k)}} [d P^{(k)}_{t}(W) - d \wtilde P^{(k)}_{t}(W)]\\
    &=C_4\cdot {\text{TV}}(P^{(k)}_{t} || \wtilde P^{(k)}_{t} ).
\end{align*}
\end{proof}



With \pref{lem: bound_using_dtv_app}, we are now ready to derive bounds for bias gap and variance gap. The bias gap is illustrated in the following lemma.

\begin{lem}\label{lem: bias_gap_bound}
Let $W^*$ denote the optimal model parameter, i.e., $W^* = \arg\min_{W}\frac{1}{N} \sum_{i = 1}^N \calL(W, d_i)$. Let $\text{Bias}(W_{t}^{\calRO}) = \E[\|\E[W_{t}^{\calRO}] - W^*\|]$. We have that 
\begin{align*}
    \big|\text{Bias}(\E[\tildewt|W_{t - 1}]) - \text{Bias}(W_{t}^{\calRO})\big|&\le C_4\cdot {\text{TV}}(P_t^{(k)} || \wtilde P_t^{(k)} ),
\end{align*}
where $W_{t}^{\calRO} = W_{t-1}^{\calRO} - \frac{1}{|\calD^{(k)}|} \sum_{i = 1}^{\calD^{(k)}} \nabla\calL(W_{t-1}^{\calRO}, d_i^{(k)})$, and $\widetilde W_{t}^{(k)} = W_{t}^{(k)} + \delta_{t}^{(k)}$.
\end{lem}

\begin{proof}
Recall that $\text{Bias}(W_{t}^{\calRO}) = \E[\|\E[W_{t}^{\calRO}] - W^*\|]$. Therefore, we have that
\begin{align}\label{eq: gap_of_bias}
    \big|\text{Bias}(\E[\tildewt|W_{t - 1}]) - \text{Bias}(W_{t}^{\calRO})\big| 
    & = \big|\E[\|\E[W_{t}^{\calRO}] - W^*\|] - \E[\|\E[\E[\tildewt|W_{t - 1}]] - W^*\|]\big|\nonumber\\
    &\le C_4\cdot {\text{TV}}(P^{(k)}_{t} || \wtilde P^{(k)}_{t} ),
\end{align}
where the inequality is due to \pref{lem: bound_using_dtv_app}.
\end{proof}

\section{Analysis for Variance Gap (Lemma \ref{lem: variance_gap_bound})}\label{app: variance_gap_bound_app}

The variance gap is illustrated in the following lemma.
\begin{lem}[Variance Gap]\label{lem: variance_gap_bound_app}
Let $\text{Var}[W_{t}^{\calRO}] = \E[\|W_{t}^{\calRO} - \E[W_{t}^{\calRO}]\|^2]$. We have that
\begin{align}
    \text{Var}(\E[\tildewt|W_{t - 1}])  - \text{Var}(W_{t}^{\calRO})\le \underbrace{-\E(\text{Var}[\tildewt| W_{t - 1}])}_{\text{variance reduction}} + 2\sup\|W\|_2C_3\cdot {\text{TV}}(P^{(k)}_{t} || \wtilde P^{(k)}_{t} ).
\end{align}
\end{lem}
\begin{proof}
From the law of total variance, we have that
\begin{align}\label{eq: variance_decomposition_1}
    \text{Var}(W_{t}^{\calRO}) = \underbrace{\E(\text{Var}[W_{t}^{\calRO}| W_{t - 1}])}_{\text{average within sample variance}} +  \underbrace{\text{Var}(\E[W_{t}^{\calRO}| W_{t - 1}])}_{\text{between sample variance}}.
\end{align}

Therefore, we have that
\begin{align*}
    &\text{Var}(\E[\tildewt|W_{t - 1}])  - \text{Var}(W_{t}^{\calRO})\\ 
    & = (\text{Var}(\tildewt) -\E(\text{Var}[\tildewt| W_{t - 1}])) - \text{Var}(W_{t}^{\calRO})\\
    & =  \underbrace{-\E(\text{Var}[\tildewt| W_{t - 1}])}_{\text{variance reduction}} + \text{Var}(\tildewt) - \text{Var}(W_{t}^{\calRO})\\
    & = \underbrace{-\E(\text{Var}[\tildewt| W_{t - 1}])}_{\text{variance reduction}} + 
    (\E {\tildewt} (\tildewt)^{\T} - \E {\tildewt} \E{(\tildewt)^{\T}}) - \left(\E W_{t}^{\calRO} (W_{t}^{\calRO})^{\T} - \E {W_{t}^{\calRO}} \E (W_{t}^{\calRO})^{\T}\right) \\
    & = \underbrace{-\E(\text{Var}[\tildewt| W_{t - 1}])}_{\text{variance reduction}} + 
    (\E {\tildewt} (\tildewt)^{\T} - \E W_{t}^{\calRO} (W_{t}^{\calRO})^{\T}) + (\E {W_{t}^{\calRO}} \E(W_{t}^{\calRO})^{\T} - \E {\tildewt} \E{(\tildewt)^{\T}})\\
    & = \underbrace{-\E(\text{Var}[\tildewt| W_{t - 1}])}_{\text{variance reduction}} + \Delta_1 + \Delta_2,
\end{align*}
where the first equation is due to $\text{Var}(\E[\tildewt|W_{t - 1}]) = \text{Var}(\tildewt) -\E(\text{Var}[\tildewt| W_{t - 1}])$ from \pref{eq: variance_decomposition_1}, $\Delta_1 = \E {\tildewt} (\tildewt)^{\T} - \E W_{t}^{\calRO} (W_{t}^{\calRO})^{\T}$, and $\Delta_2 = \E {W_{t}^{\calRO}} \E(W_{t}^{\calRO})^{\T} - \E {\tildewt} \E{(\tildewt)^{\T}}$. Now we bound the variance term $\text{tr}(\Delta_1) + \text{tr}(\Delta_2)$.


\begin{align*}
    |\text{tr}(\Delta_1)| &= \big|\E_{\sigma_t^{(k)}}\E_{W_{t}^{\calRO}}[\|\tildewt\|_2^2 - \|W_{t}^{\calRO}\|_2^2]\big|\\
    &\le 2\sup\|W\|_2\E\big|\E[\|\tildewt\|_2 - \|W_{t}^{\calRO}\|_2]\big|\\
    &\le 2\sup\|W\|_2\cdot C_3\cdot {\text{TV}}(P^{(k)}_{t} || \wtilde P^{(k)}_{t} ),
\end{align*}
where the second inequality is due to \pref{lem: bound_using_dtv_norm_of_para_app}.

Recall that the distorted model parameter $\widetilde W_{t}^{(k)} = W_{t}^{(k)} + \delta_{t}^{(k)}$. We also have that

\begin{align*}
    |\text{tr}(\Delta_2)| &= \big|\|\E_{W_{t}^{\calRO}}[W_{t}^{\calRO}]\|_2^2 - \|\E_{\sigma_t^{(k)}}\E_{W_{t}^{\calRO}}[\tildewt]\|_2^2\big|\\
    & = 0.
\end{align*}

Therefore, we have that
\begin{align}
    \text{Var}(\E[\tildewt|W_{t - 1}])  - \text{Var}(W_{t}^{\calRO})\le \underbrace{-\E(\text{Var}[\tildewt| W_{t - 1}])}_{\text{variance reduction}} + 2\sup\|W\|_2C_3\cdot {\text{TV}}(P^{(k)}_{t} || \wtilde P^{(k)}_{t} ).
\end{align}
\end{proof}

\section{Analysis for Lemma \ref{lem: privacy_leakage_upper_bound}}

The following lemma illustrates that the privacy leakage could be upper bounded by the total variation distance between $P^{\calRO}_t$ and $\wtilde P^{(k)}$.

\begin{lem}[Upper Bound for Privacy Leakage]\label{lem: privacy_leakage_upper_bound_app}
Let $F^{(k)}_t$ and $\wtilde F^{(k)}_t$ represent the belief of client $k$ about $S$ after observing the original parameter and the protected parameter. Let $C_{1,t}^{(k)} = \sqrt{{\text{JS}}(\breve F^{(k)}_t || F^{(k)}_t)}$, and $C_2 = \frac{1}{2}(e^{2\xi}-1)$, where $\xi = \max_{k\in [K]} \xi^{(k)}$, $\xi^{(k)} = \max_{w\in \mathcal{W}^{(k)}, d \in \mathcal{D}^{(k)}} \left|\log\left(\frac{f_{D^{(k)}|W^{(k)}}(d|w)}{f_{D^{(k)}}(d)}\right)\right|$ represents the maximum privacy leakage over all possible information $w$ released by client $k$, and $[K] = \{1,2,\cdots, K\}$. Let $P^{\calRO}_t$ and $\wtilde P_t^{(k)}$ represent the distribution of the parameter of client $k$ at round $t$ before and after being protected. Assume that $C_2\cdot{\text{TV}}(P^{\calRO}_t || \wtilde P_t^{(k)})\le C_{1,t}^{(k)}$. The upper bound for the privacy leakage of client $k$ is
\begin{align*}
    \epsilon_{p,t}^{(k)} \le 2C_{1,t}^{(k)} - C_2\cdot {\text{TV}}(P^{\calRO}_t || \wtilde P_t^{(k)}).
\end{align*}
\end{lem}

\begin{proof}

Notice that the square root of the Jensen-Shannon divergence satisfies the triangle inequality. $\xi = \max_{k\in [K]} \xi^{(k)}$, $\xi^{(k)} = \max_{w\in \mathcal{W}^{(k)}, d \in \mathcal{D}^{(k)}} \left|\log\left(\frac{f_{D^{(k)}|W^{(k)}}(d|w)}{f_{D^{(k)}}(d)}\right)\right|$ represents the maximum privacy leakage over all possible information $w$ released by client $k$, and $[K] = \{1,2,\cdots, K\}$. Fixing the attacking extent, then $\xi$ is a constant.

Then we have that

\begin{align*}
    \epsilon_{p,t}^{(k)} = \sqrt{{\text{JS}}(\wtilde F_t^{(k)} || \breve F^{(k)}_t)}&\le\sqrt{{\text{JS}}(F^{(k)}_t || \breve F^{(k)}_t)} + \sqrt{{\text{JS}}(\wtilde F_t^{(k)} || F^{(k)}_t)}\\
    &\le\sqrt{{\text{JS}}(F^{(k)}_t || \breve F^{(k)}_t)} + \frac{1}{2}\cdot(e^{2\xi}-1){\text{TV}}(P^{\calRO}_t || \wtilde P^{(k)}_t)\\
    &\le 2\sqrt{{\text{JS}}(F^{(k)}_t || \breve F^{(k)}_t)} - \frac{1}{2}\cdot(e^{2\xi}-1){\text{TV}}(P^{\calRO}_t || \wtilde P^{(k)}_t)\\
    & = 2C_{1,t}^{(k)} - C_2\cdot{\text{TV}}(P^{\calRO}_t || \wtilde P^{(k)}_t),
\end{align*}
where the second inequality is due to \pref{lem: JSBound}, and 
the third inequality is due to the assumption that $C_2\cdot{\text{TV}}(P^{\calRO}_t || \wtilde P^{(k)}_t)\le \sqrt{{\text{JS}}(F^{(k)}_t || \breve F^{(k)}_t)}$.
\end{proof}

\section{Analysis for Lemma \ref{lem: from_privacy_to_distortion}}\label{sec: from_privacy_to_distortion}
\begin{lem}
Let $C_{1,t} =\frac{1}{K}\sum_{k=1}^K \sqrt{{\text{JS}}(F^{(k)}_t || \breve F^{(k)}_t)}$. If the total variation distance is at least
\begin{align}
    \TV(P_t^{(k)}||\wtilde P_t^{(k)})\ge C_{1,t} - \tau_{p,t}^{(k)},
\end{align}
then the privacy leakage $\epsilon_{p,t}^{(k)}$ is at most $\tau_{p,t}^{(k)}$.
\end{lem}
\begin{proof}
Given the requirement that the privacy leakage of client $k$ should not exceed the threshold $\tau_{p,t}^{(k)}$. 
\begin{align}
    \epsilon_{p,t}^{(k)} \le C_{1,t} - \TV(P_t^{(k)}||\wtilde P_t^{(k)}).
\end{align}
When
\begin{align}
    \TV(P_t^{(k)}||\wtilde P_t^{(k)})\ge C_{1,t} - \tau_{p,t}^{(k)}, 
\end{align}
we have
\begin{align}
    \epsilon_{p,t}^{(k)} \le C_{1,t} - \TV(P_t^{(k)}||\wtilde P_t^{(k)})\le \tau_{p,t}^{(k)},
\end{align}
where the first inequality is due to the upper bound of privacy leakage derived in \pref{lem: privacy_leakage_upper_bound}.
\end{proof}

\section{Analysis for Lemma \ref{lem: variance_and_privacy_leakage}}\label{sec: variance_and_privacy_leakage}
\begin{lem}\label{lem: variance_and_privacy_leakage_app}
Assume that $0< C_{1,t} - \tau_{p,t}^{(k)} < 0.01$. Let $\sigma^2$ represent the variance of the original model parameter, and $\sigma_\epsilon^2$ represent the variance of the added noise. If the variance of the added noise $\sigma_\epsilon^2 = \frac{100\sigma^2(C_{1,t} - \tau_{p,t}^{(k)})}{\sqrt{d}}$, then the privacy leakage $\epsilon_{p,t}^{(k)}$ is at most $\tau_{p,t}^{(k)}$.
\end{lem}
\begin{proof}
Let
\begin{align}
    \sigma_\epsilon^2 = \frac{100\sigma^2(C_{1,t} - \tau_{p,t}^{(k)})}{\sqrt{d}}.
\end{align}

Then
\begin{align}
    \frac{1}{100}\min\left\{1, \frac{\sigma_\epsilon^2\sqrt{d}}{\sigma^2} \right\} = \frac{\sigma_\epsilon^2\sqrt{d}}{100\sigma^2} \ge C_{1,t} - \tau_{p,t}^{(k)}.
\end{align}
From \pref{lem: total_variation_and_variance}, we know that  
\begin{align}
     {\text{TV}}(P^{\calRO} || \wtilde P^{(k)} )\ge \frac{1}{100}\min\left\{1, \frac{\sigma_\epsilon^2\sqrt{d}}{\sigma^2} \right\}\ge C_{1,t} - \tau_{p,t}^{(k)}.
\end{align}

If the total variation distance $\TV(P_t^{(k)}||\wtilde P_t^{(k)})\ge C_{1,t} - \tau_{p,t}^{(k)}$, then the privacy leakage $\tau_{p,t}^{(k)}$ is at most $\tau_{p,t}^{(k)}$ from \pref{lem: from_privacy_to_distortion}, where $C_{1,t} =\frac{1}{K}\sum_{k=1}^K \sqrt{{\text{JS}}(F^{(k)}_t || \wtilde F^{(k)}_t)}$.
\end{proof}

\section{Analysis for Lemma \ref{lem: conditional_expectation_of_model_parameter}}\label{sec: conditional_expectation_of_model_parameter}
The following lemma calculates the expectation of the model parameter $\distpara$.
\begin{lem}\label{lem: conditional_expectation_of_model_parameter_app}
Let $\distpara = W_{t-1}^{\calRO} - \frac{1}{N} \sum_{j = 1}^N \sum_{i = 1}^M \nabla\calL(W_{t-1}^{\calRO}, d_i^{(k)})\one\{d_i ^{(k)}\text{ is selected at } $j-$\text{th} \text{ round}\} + \delta_{t - 1}^{(k)}.$ Let $M$ represent the data size, and $N$ represent the total number of rounds for sampling. We have that
\begin{align}
    \E[\distpara] = W_{t-1}^{\calRO} - p\cdot\sum_{i = 1}^M \nabla\calL(W_{t-1}^{\calRO}, d_i^{(k)}) + \delta_{t - 1}^{(k)}.
\end{align}
\end{lem}
\begin{proof}
The update rule of the distorted model parameter is
\begin{align}
    \distpara = W_{t-1}^{\calRO} - \frac{1}{N} \sum_{j = 1}^N \sum_{i = 1}^M \nabla\calL(W_{t-1}^{\calRO}, d_i^{(k)})\one\{d_i ^{(k)}\text{ is selected at } j\text{th} \text{ round}\} + \delta_{t - 1}^{(k)}.
\end{align}

Recall that the sampling model is as follows:
\begin{itemize}
    \item At the $i$-th iteration, each $d\in\calD^{(k)}$ is sampled with probability $p$;
    \item The total number of iterations is $N$.
\end{itemize}

Each data $d_i^{(k)}$ is sampled with probability $p$. For any $d_i^{(k)}$, we have that 
\begin{align}
    \E[\sum_{i = 1}^M \nabla\calL(W_{t-1}^{\calRO}, d_i^{(k)})\one\{d_i ^{(k)}\text{ is selected at } j\text{th} \text{ round}\}] = p\cdot\sum_{i = 1}^{M} \nabla\calL(W_{t-1}^{\calRO}, d_i^{(k)}).
\end{align}

Therefore, we have
\begin{align}
   &\E\left[\frac{1}{N} \sum_{j = 1}^N \sum_{i = 1}^M \nabla\calL(W_{t-1}^{\calRO}, d_i^{(k)})\one\{d_i ^{(k)}\text{ is selected at } j\text{th} \text{ round}\}\right]\\ 
   &= \E\left[\sum_{i = 1}^M \nabla\calL(W_{t-1}^{\calRO}, d_i^{(k)})\one\{d_i ^{(k)}\text{ is selected at } j\text{th} \text{ round}\}\right]\\
   &= p\cdot\sum_{i = 1}^{M} \nabla\calL(W_{t-1}^{\calRO}, d_i^{(k)}),
\end{align}
where $N$ represents the total number of rounds for sampling, and $M$ represents the data size.

Therefore, we have that
\begin{align}
    \E[\distpara] = W_{t-1}^{\calRO} - p\cdot\sum_{i = 1}^M \nabla\calL(W_{t-1}^{\calRO}, d_i^{(k)}) + \delta_{t - 1}^{(k)}.
\end{align}
\end{proof}

\section{Analysis for Theorem \ref{thm: variance_of_distorted_model_parameter}}\label{app: variance_of_distorted_model_parameter}
The following theorem calculates the variance of the model parameter $\distpara$. Fixing $W_{t-1}^{\calRO}$ and data $d_i$, then $\text{Var}[\distpara]$ depends on $p$.

\begin{thm}\label{thm: conditional_variance_of_model_parameter}
We denote $p$ as the sampling probability. That is, each data of each client is sampled with probability $p$ to generate the batch. Let 
\begin{align}
   \distpara = W_{t-1}^{\calRO} - \frac{1}{N} \sum_{j = 1}^N \sum_{i = 1}^M \nabla\calL(W_{t-1}^{\calRO}, d_i^{(k)})\one\{d_i ^{(k)}\text{ is selected at } j\text{th} \text{ round}\} + \delta_{t - 1}^{(k)}.
\end{align}
Let $M$ represent the data size, and $N$ represent the total number of rounds for sampling. We have that
\begin{align}
    \text{Var}[\distpara| W_{t - 1}] = p\cdot (1-p)\cdot\sum_{i = 1}^M \left(\nabla\calL(W_{t-1}^{\calRO}, d_i^{(k)})\right)^2.
\end{align}
\end{thm}
\begin{proof}

Recall that
\begin{align}
    \distpara = W_{t-1}^{\calRO} - \frac{1}{N} \sum_{j = 1}^N \sum_{i = 1}^M \nabla\calL(W_{t-1}^{\calRO}, d_i^{(k)})\one\{d_i ^{(k)}\text{ is selected at } j\text{th} \text{ round}\} + \delta_{t - 1}^{(k)}.
\end{align}

From \pref{lem: conditional_expectation_of_model_parameter_app}, we know that
\begin{align}
    \E[\distpara] = W_{t-1}^{\calRO} - p\cdot\sum_{i = 1}^M \nabla\calL(W_{t-1}^{\calRO}, d_i^{(k)}) + \delta_{t - 1}^{(k)}.
\end{align}
Recall that the sampling model is as follows:
\begin{itemize}
    \item At the $i$-th iteration, each $d\in\calD^{(k)}$ is sampled with probability $p$;
    \item The total number of iterations is $N$.
\end{itemize}
Each data $d_i$ is sampled with probability $p$. We have that
\begin{align*}
    &\text{Var}[\distpara|W_{t - 1}]\\ 
    & = \E\left[\left(\distpara - \E[\distpara]\right)^2|W_{t - 1}\right] \\
     &=\E[(\frac{1}{N} \sum_{j = 1}^N \sum_{i = 1}^M \nabla\calL(W_{t-1}^{\calRO}, d_i^{(k)})\one\{d_i ^{(k)}\text{ is selected at } j\text{th} \text{ round}\}\\
     & - \E[\frac{1}{N} \sum_{j = 1}^N \sum_{i = 1}^M \nabla\calL(W_{t-1}^{\calRO}, d_i^{(k)})\one\{d_i ^{(k)}\text{ is selected at } j\text{th} \text{ round}\}])^2|W_{t - 1}]\\
     & = \text{Var}\left[\frac{1}{N} \sum_{j = 1}^N \sum_{i = 1}^M \nabla\calL(W_{t-1}^{\calRO}, d_i^{(k)})\one\{d_i ^{(k)}\text{ is selected at } j\text{th} \text{ round}\}|W_{t - 1}\right]\\
     & = \frac{1}{N}\text{Var}\left[\sum_{i = 1}^M \nabla\calL(W_{t-1}^{\calRO}, d_i^{(k)})\one\{d_i ^{(k)}\text{ is selected at } j\text{th} \text{ round}\}|W_{t - 1}\right]\\
     & = \frac{1}{N}\sum_{i = 1}^M \left(\nabla\calL(W_{t-1}^{\calRO}, d_i^{(k)})\right)^2\text{Var}\left[\one[d_i ^{(k)}\text{ is selected}]|W_{t - 1}\right]\\
     & = \frac{1}{N}\cdot p\cdot (1-p)\cdot\sum_{i = 1}^M \left(\nabla\calL(W_{t-1}^{\calRO}, d_i^{(k)})\right)^2, 
\end{align*}
where the second equality is due to \pref{lem: conditional_expectation_of_model_parameter_app}.
\end{proof}

\section{Analysis for Theorem \ref{thm: upper_bound_of_epsilon_u}}\label{sec: upper_bound_of_epsilon_u}
In this section, we introduce our main theorem, which illustrates the condition for achieving near-optimal utility.

The following theorem shows that the utility loss is bounded by the distance between the protected and unprotected distributions. We provide an upper bound for utility loss using the property of sum of squares and bias-variance decomposition. This theorem informs how to obtain near-optimal utility. When $\E(\text{Var}[\tildewt| W_{t - 1}]) = C_6\cdot {\text{TV}}(P^{(k)}_{t} || \wtilde P^{(k)}_{t} )$, the utility loss is $0$.

\begin{thm}[Upper Bounds for Utility Loss]
Let $\epsilon_{u,t}^{(k)}$ be defined in \pref{defi: utility_loss}, then we have that
\begin{align}
    \epsilon_{u,t}^{(k)} \le -\E(\text{Var}[\tildewt| W_{t - 1}]) + C_6\cdot {\text{TV}}(P^{(k)}_{t} || \wtilde P^{(k)}_{t} ),
\end{align}
where the first term is related to generalization and corresponds to the stochastic gradient descent procedure, and the second term is related to the protection mechanism.
\end{thm}

\begin{proof}

From \pref{lem: variance_bias_decomp}, we have that
\begin{align*}
    \text{GAP}(W_{t}^{\calRO})=\underbrace{\text{tr}(\text{Var}[W_{t}^{\calRO}])}_{\textbf{variance}} + \underbrace{\text{Bias}^2(W_{t}^{\calRO})}_{\textbf{bias}}. 
\end{align*}
Therefore, we have that

\begin{align*}
    \epsilon_{u,t}^{(k)} & = \text{GAP}(\tildewt) - \text{GAP}(W_{t}^{\calRO}) \\
    & = (\text{Var}(\E[\tildewt|W_{t - 1}]) + \text{Bias}^2(\E[\tildewt|W_{t - 1}])) - (\text{Var}[W_{t}^{\calRO}] + \text{Bias}^2(W_{t}^{\calRO}))\\
    & \le  (\text{Var}(\E[\tildewt|W_{t - 1}]) - \text{Var}[W_{t}^{\calRO}]) + |\text{Bias}^2(W_{t}^{\calRO}) - \text{Bias}^2(\E[\tildewt|W_{t - 1}])|.
\end{align*} 
First we provide bounds for the gap of the bias. 
\paragraph{Bounding Bias Gap}

\begin{align*}
    &\left|\text{Bias}^2(W_{t}^{\calRO}) - \text{Bias}^2(\E[\tildewt|W_{t - 1}])\right|\\ 
    & = \left(\text{Bias}(\E[\tildewt|W_{t - 1}]) + \text{Bias}(W_{t}^{\calRO})\right)\cdot \big|\text{Bias}(\E[\tildewt|W_{t - 1}]) - \text{Bias}(W_{t}^{\calRO})\big|\\
    &\le \left(\big|\text{Bias}(\E[\tildewt|W_{t - 1}]) - \text{Bias}(W_{t}^{\calRO})\big| + 2\text{Bias}(W_{t}^{\calRO})\right)\cdot\big|\text{Bias}(\E[\tildewt|W_{t - 1}]) - \text{Bias}(W_{t}^{\calRO})\big|.
\end{align*}
From \pref{lem: bias_gap_bound}, we have that 

\begin{align*}
   \big|\text{Bias}(\E[\tildewt|W_{t - 1}]) - \text{Bias}(W_{t}^{\calRO})\big|&\le C_4\cdot {\text{TV}}(P_t^{(k)} || \wtilde P_t^{(k)} ).
\end{align*}

Therefore, we have
\begin{align*}
    &\left|\text{Bias}^2(W_{t}^{\calRO}) - \text{Bias}^2(\E[\tildewt|W_{t - 1}])\right|\\ 
    & \le (C_3\cdot {\text{TV}}(P_t^{(k)} || \wtilde P_t^{(k)} )  + 2\text{Bias}(W_{t}^{\calRO}))\cdot C_4\cdot {\text{TV}}(P_t^{(k)} || \wtilde P_t^{(k)} ).
\end{align*}

\paragraph{Bounding Variance}

From \pref{lem: variance_gap_bound_app}, we have that

\begin{align}
    \text{Var}(\E[\tildewt|W_{t - 1}])  - \text{Var}(W_{t}^{\calRO})\le \underbrace{-\E(\text{Var}[\tildewt| W_{t - 1}])}_{\text{variance reduction}} + 2\sup\|W\|_2C_3\cdot {\text{TV}}(P^{(k)}_{t} || \wtilde P^{(k)}_{t} ).
\end{align}
For facility of expression, we assume that $\text{Bias}(W_{t}^{\calRO})$ is very small, and satisfies that $\text{Bias}(W_{t}^{\calRO})\le C_5\cdot {\text{TV}}(P^{(k)}_{t} || \wtilde P^{(k)}_{t} )$, where $C_5 > 0$ represents a constant.
Therefore, we have
\begin{align}
    \epsilon_{u,t}^{(k)} \le -\E(\text{Var}[\tildewt| W_{t - 1}]) + C_6\cdot {\text{TV}}(P^{(k)}_{t} || \wtilde P^{(k)}_{t} ),
\end{align}
where $C_6 > 0$ represents a constant, and we use the property that $\text{TV}(\cdot)\le 1$.
\end{proof}

\section{Analysis for Theorem \ref{thm: from _utility_to_probability}}
\begin{thm} 
   Given the requirement that the privacy leakage $\epsilon_{p,t}^{(k)}$ should not exceed $\tau_{p,t}^{(k)}$. If the sampling probability $p$ satisfies
\begin{align}\label{eq: sample_prob_equation_app_3}
    p(1-p)\ge \frac{C_6\cdot(C_{1,t} - \tau_{p,t}^{(k)})}{\sum_{i = 1}^M \left(\nabla\calL(W_{t-1}^{\calRO}, d_i)\right)^2},
\end{align}  
   then client $k$ achieves near-optimal utility.
\end{thm}

\begin{proof}
The first term of \pref{eq: bound_for_utility_loss} represents the variance, and the second term represents the bias. From \pref{thm: upper_bound_of_epsilon_u}, we know that
\begin{align}
    \epsilon_{u,t}^{(k)} \le -\E(\text{Var}[\distpara]) + C_6\cdot\TV(P_t^{(k)}||\wtilde P_t^{(k)}).
\end{align}
From \pref{lem: from_privacy_to_distortion}, we have that
\begin{align}
    \TV(P_t^{(k)}||\wtilde P_t^{(k)})\ge C_{1,t} - \tau_{p,t}^{(k)}. 
\end{align}
From \pref{thm: variance_of_distorted_model_parameter}, we know that the variance of the distorted model parameter $\distpara$ is
\begin{align}
    \text{Var}[\distpara] = p\cdot (1-p)\cdot\sum_{i = 1}^M \left(\nabla\calL(W_{t-1}^{\calRO}, d_i)\right)^2.
\end{align}
If
\begin{align*}
    p\cdot (1-p)\cdot\sum_{i = 1}^M \left(\nabla\calL(W_{t-1}^{\calRO}, d_i)\right)^2 &\ge C_6\cdot\TV(P_t^{(k)}||\wtilde P_t^{(k)})\\
    &\ge C_6\cdot(C_{1,t} - \tau_{p,t}^{(k)}).
\end{align*}
Then, we have
\begin{align}
    -\E(\text{Var}[\distpara]) + C_6\cdot\TV(P_t^{(k)}||\wtilde P_t^{(k)})\le 0.
\end{align}
Therefore, when the sampling probability $p$ satisfies
\begin{align}\label{eq: sample_prob_equation_app_2}
    p(1-p)\ge \frac{C_6\cdot(C_{1,t} - \tau_{p,t}^{(k)})}{\sum_{i = 1}^M \left(\nabla\calL(W_{t-1}^{\calRO}, d_i)\right)^2},
\end{align}
client $k$ achieves near-optimal utility (we set the sampling probability $p$ as the minimal optional value).
\end{proof}


\section{Analysis for Theorem \ref{thm: from _utility_to_probability}}
\begin{thm} 
   Given the requirement that the privacy leakage $\epsilon_{p,t}^{(k)}$ should not exceed $\tau_{p,t}^{(k)}$. If the sampling probability $p$ satisfies
\begin{align}\label{eq: sample_prob_equation_app}
    p(1-p)\ge \frac{C_6\cdot(C_{1,t} - \tau_{p,t}^{(k)})}{\sum_{i = 1}^M \left(\nabla\calL(W_{t-1}^{\calRO}, d_i)\right)^2},
\end{align}  
   then client $k$ achieves near-optimal utility.
\end{thm}
\begin{proof}
The first term of \pref{eq: bound_for_utility_loss} represents the variance, and the second term represents the bias. From \pref{thm: upper_bound_of_epsilon_u}, we know that
\begin{align}
    \epsilon_{u,t}^{(k)} \le -\E(\text{Var}[\distpara| W_{t - 1}]) + C_4\cdot\TV(P_t^{(k)}||\wtilde P_t^{(k)}).
\end{align}
From \pref{lem: from_privacy_to_distortion}, we have that
\begin{align}
    \TV(P_t^{(k)}||\wtilde P_t^{(k)})\ge C_{1,t} - \tau_{p,t}^{(k)}. 
\end{align}
Denote $W_{t}^{\calRO} = W_{t-1} - \frac{1}{|\calD^{(k)}|} \sum_{i = 1}^{\calD^{(k)}} \nabla\calL(W_{t-1}^{\calRO}, d_i^{(k)})$. Recall that
\begin{align}
    \distpara = W_{t-1}^{\calRO} - \frac{1}{N} \sum_{j = 1}^N \sum_{i = 1}^M \nabla\calL(W_{t-1}^{\calRO}, d_i^{(k)})\one\{d_i ^{(k)}\text{ is selected at } j-\text{th} \text{ round}\} + \delta_{t - 1}^{(k)}.
\end{align}
From \pref{thm: variance_of_distorted_model_parameter}, we know that the variance of the distorted model parameter $\distpara$ is
\begin{align}
    \text{Var}[\distpara| W_{t - 1}] = p\cdot (1-p)\cdot\sum_{i = 1}^M \left(\nabla\calL(W_{t-1}^{\calRO}, d_i^{(k)})\right)^2.
\end{align}
If
\begin{align*}
    p\cdot (1-p)\cdot\sum_{i = 1}^M \left(\nabla\calL(W_{t-1}^{\calRO}, d_i)\right)^2 &\ge C_4\cdot\TV(P_t^{(k)}||\wtilde P_t^{(k)})\\
    &\ge C_4\cdot(C_{1,t} - \tau_{p,t}^{(k)}).
\end{align*}
Then, we have
\begin{align}
    -\E(\text{Var}[\distpara| W_{t - 1}]) + C_6\cdot\TV(P_t^{(k)}||\wtilde P_t^{(k)})\le 0.
\end{align}
Therefore, when the sampling probability $p$ satisfies
\begin{align}\label{eq: sample_prob_equation_app}
    p(1-p)\ge \frac{C_6\cdot(C_{1,t} - \tau_{p,t}^{(k)})}{\sum_{i = 1}^M \left(\nabla\calL(W_{t-1}^{\calRO}, d_i)\right)^2},
\end{align}
client $k$ achieves near-optimal utility (we set the sampling probability $p$ as the minimal optional value).
\end{proof}

\section{Analysis for Optimal Trade-off}\label{sec: bounds_for_trade_off}
\subsection{Analysis for Theorem \ref{thm: upper_bound_trade_off_mt}}
\begin{thm}[Upper Bound for Trade-off]\label{thm: upper_bound_trade_off_app}
  Let \pref{assump: upper_bounded_by_c_3} and \pref{assump: upper_bounded_by_c_4} hold. We have that
  \begin{align*}
     \epsilon_{p,t}^{(k)} + \frac{C_2}{C_4}\cdot\epsilon_{u,t}^{(k)} &\le -\frac{C_2}{C_6}\cdot\E(\text{Var}[\distpara| W_{t - 1}]) + 2 C_{1,t}^{(k)}.
\end{align*}
where $C_{1,t}^{(k)} = \sqrt{{\text{JS}}(F^{(k)}_t || \breve F^{(k)}_t)}$, $C_2$ is introduced in \pref{eq: c_2_definition}, and $C_6$ is introduced in \pref{thm: upper_bound_of_epsilon_u}.
\end{thm}

\begin{proof}
   From \pref{thm: upper_bound_of_epsilon_u}, the utility loss is upper bounded by
\begin{align}
    \epsilon_{u,t}^{(k)} \le -\E(\text{Var}[\distpara| W_{t - 1}]) + C_6\cdot\TV(P_t^{(k)}||\wtilde P_t^{(k)}).
\end{align}
From \pref{lem: privacy_leakage_upper_bound}, the relationship between the total variation distance and the privacy leakage is
\begin{align}
    \epsilon_{p,t}^{(k)} \le 2C_{1,t}^{(k)} - C_2\cdot {\text{TV}}(P^{\calRO}_t || \wtilde P_t^{(k)}).
\end{align}
Therefore, we have
\begin{align*}
\epsilon_{u,t}^{(k)} &\le -\E(\text{Var}[\distpara| W_{t - 1}]) + C_6\cdot {\text{TV}}(P_t^{(k)} || \wtilde P_t^{(k)} )\\
&\le -\E(\text{Var}[\distpara| W_{t - 1}]) + \frac{C_6}{C_2}\cdot (2 C_{1,t}^{(k)} - \epsilon_{p,t}^{(k)}).
\end{align*}

\end{proof}






\subsection{Analysis for Theorem \ref{thm: optimal_trade_off}}

Let $\epsilon_{p,t}^{(k)}$ be defined in \pref{defi: average_privacy_JSD}, let $\epsilon_{u,t}^{(k)}$ be defined in \pref{defi: utility_loss}, and let \pref{assump: assump_of_Delta} hold. From No free lunch theorem (NFL) for privacy and utility, we have that $\epsilon_{p,t}^{(k)} + C_d\cdot \epsilon_{u, t}^{(k)}\ge C_{1,t}^{(k)}$.

\begin{thm}[Lower Bound for Trade-off, see Theorem 4.1 of \cite{zhang2022no}]\label{thm: trade_off_lower_bound} 
Let $\epsilon_{p,t}^{(k)}$ be defined in \pref{defi: average_privacy_JSD}, and let $\epsilon_{u,t}^{(k)}$ be defined in \pref{defi: utility_loss}, with \pref{assump: assump_of_Delta} we have:
\begin{align}
 \epsilon_{p,t}^{(k)} + C_d\cdot \epsilon_{u,t}^{(k)}\ge C_{1,t}^{(k)},
\end{align}
where $C_{1,t}^{(k)} = \sqrt{{\text{JS}}(F^{(k)}_t || \breve F^{(k)}_t)}$, $C_d = \frac{\gamma}{4\Delta}(e^{2\xi}-1)$, where $\xi^{(k)}$=$\max_{w\in \mathcal{W}^{(k)}, d \in \mathcal{D}^{(k)}} \left|\log\left(\frac{f_{D^{(k)}|W^{(k)}}(d|w)}{f_{D^{(k)}}(d)}\right)\right|$, $\xi$=$\max_{k\in [K]} \xi^{(k)}$ represents the maximum privacy leakage over all possible information $w$ released by client $k$, and $\Delta$ is introduced in \pref{assump: assump_of_Delta}.
\end{thm}

In certain scenarios, we provide the condition for achieving optimal trade-off.      
\begin{thm}[Optimal Trade-off]
Consider the scenario where $C_d = \frac{C_2}{C_6}$. If $C_{1,t} = \frac{C_2}{C_6}\cdot\E(\text{Var}[\distpara| W_{t - 1}])$, then the optimal trade-off is achieved, where $C_{1,t}^{(k)} = \sqrt{{\text{JS}}(F^{(k)}_t || \breve F^{(k)}_t)}$, $C_d = \frac{\gamma}{4\Delta}(e^{2\xi}-1)$, where $\xi^{(k)}$=$\max_{w\in \mathcal{W}^{(k)}, d \in \mathcal{D}^{(k)}} \left|\log\left(\frac{f_{D^{(k)}|W^{(k)}}(d|w)}{f_{D^{(k)}}(d)}\right)\right|$, $\xi$=$\max_{k\in [K]} \xi^{(k)}$ represents the maximum privacy leakage over all possible information $w$ released by client $k$, and $\Delta$ is introduced in \pref{assump: assump_of_Delta}, and $C_6$ is introduced in \pref{thm: upper_bound_of_epsilon_u}.
\end{thm}
\begin{proof}
From \pref{thm: upper_bound_trade_off_app}, we have
\begin{align*}
     \epsilon_{p,t}^{(k)} + \frac{C_2}{C_6}\cdot\epsilon_{u,t}^{(k)} &\le -\frac{C_2}{C_6}\cdot\E(\text{Var}[\distpara| W_{t - 1}]) + 2 C_{1,t}^{(k)}.
\end{align*}

From \pref{thm: trade_off_lower_bound}, we have
\begin{align}
 \epsilon_{p,t}^{(k)} + C_d\cdot \epsilon_{u, t}^{(k)}\ge C_{1,t}^{(k)}.
\end{align}

By setting $C_{1,t}^{(k)} = \frac{C_2}{C_6}\cdot\E(\text{Var}[\distpara| W_{t - 1}])$, the optimal trade-off is achieved.
\end{proof}

\end{document}